\documentclass[english]{article}
\usepackage[T1]{fontenc}
\usepackage[latin9]{inputenc}
\usepackage{array}
\usepackage{float}
\usepackage{textcomp}
\usepackage{url}
\usepackage{multirow}
\usepackage{amsmath}
\usepackage{amsthm}
\usepackage{amssymb}
\usepackage{graphicx}

\makeatletter

\providecommand{\tabularnewline}{\\}
\floatstyle{ruled}
\newfloat{algorithm}{tbp}{loa}
\providecommand{\algorithmname}{Algorithm}
\floatname{algorithm}{\protect\algorithmname}

\theoremstyle{plain}
\newtheorem{thm}{\protect\theoremname}
\theoremstyle{plain}
\newtheorem{prop}[thm]{\protect\propositionname}

\usepackage{microtype}
\usepackage{graphicx}
\usepackage{subfigure}
\usepackage{booktabs} 
\usepackage{hyperref}
\usepackage[final]{neurips_2021}
\usepackage{amsmath}
\usepackage{etoolbox}
\usepackage[noend]{algpseudocode}
\usepackage{amsmath}
\algnewcommand\algorithmicforeach{\textbf{for each}}
\algdef{S}[FOR]{ForEach}[1]{\algorithmicforeach\ #1\ \algorithmicdo}

\usepackage[belowskip=-3mm]{caption}

\makeatother

\usepackage{babel}
\providecommand{\propositionname}{Proposition}
\providecommand{\theoremname}{Theorem}

\begin{document}
\global\long\def\argmax{\mathrm{argmax}}%

\global\long\def\memread{\mathrm{read}}%

\global\long\def\memwrite{\mathrm{write}}%

\global\long\def\softmax{\mathrm{softmax}}%

\global\long\def\sigmoid{\mathrm{sigmoid}}%

\global\long\def\argmax{\mathrm{argmax}}%

\title{Model-Based Episodic Memory Induces Dynamic Hybrid Controls}
\author{Hung Le, Thommen Karimpanal George, Majid Abdolshah, Truyen Tran,
Svetha Venkatesh\\
Applied AI Institute, Deakin University, Geelong, Australia\\
\texttt{thai.le@deakin.edu.au}}
\maketitle
\begin{abstract}
Episodic control enables sample efficiency in reinforcement learning
by recalling past experiences from an episodic memory. We propose
a new model-based episodic memory of trajectories addressing current
limitations of episodic control. Our memory estimates trajectory values,
guiding the agent towards good policies. Built upon the memory, we
construct a complementary learning model via a dynamic hybrid control
unifying model-based, episodic and habitual learning into a single
architecture. Experiments demonstrate that our model allows significantly
faster and better learning than other strong reinforcement learning
agents across a variety of environments including stochastic and non-Markovian
settings.
\end{abstract}

\section{Introduction}

Episodic memory or ``mental time travel'' \cite{dudai2005janus}
allows recreation of past experiences. In reinforcement learning (RL),
episodic control (EC) uses this memory to control behavior, and complements
forward model and simpler, habitual (cached) control methods. The
use of episodic memory\footnote{The episodic memory in this setting is an across-lifetime memory,
persisting throughout training.} is shown to be very useful in early stages of RL \cite{lengyel2008hippocampal,blundell2016model}
and backed up by cognitive evidence \cite{tulving1972episodic,tulving2002episodic}.
Using only one or few instances of past experiences to make decisions,
EC agents avoid complicated planning computations, exploiting experiences
faster than the other two control methods. In hybrid control systems,
EC demonstrates excellent performance and better sample efficiency
\cite{pritzel2017neural,lin2018episodic}.

Early works on episodic control use tabular episodic memory storing
a raw trajectory as a sequence of states, actions and rewards over
consecutive time steps. To select a policy, the methods iterate through
all stored sequences and are thus only suitable for small-scale problems
\cite{lengyel2008hippocampal,gershman2017reinforcement}. Other episodic
memories store individual state-action pairs, acting as the state-action
value table in tabular RL, and can generalize to novel states using
nearest neighbor approximations \cite{blundell2016model,pritzel2017neural}.
Recent works \cite{nishio2018faster,hansen2018fast,lin2018episodic,zhu2019episodic}
leverage both episodic and habitual learning by combining state-action
episodic memories with Q-learning augmented with parametric value
functions like Deep Q-Network (DQN; \cite{mnih2015human}). The combination
of the ``fast'' non-parametric episodic and ``slow'' parametric
value facilitates Complementary Learning Systems (CLS) -- a theory
posits that the brain relies on both slow learning of distributed
representations (neocortex) and fast learning of pattern-separated
representations (hippocampus) \cite{mcclelland1995there}. 

Existing episodic RL methods suffer from 3 issues: (a) near-deterministic
assumption \cite{blundell2016model} which is vulnerable to noisy,
stochastic or partially observable environments causing ambiguous
observations; (b) sample-inefficiency due to storing state-action-value
which demands experiencing all actions to make reliable decisions
and inadequate memory writings that prevent fast and accurate value
propagation inside the memory \cite{blundell2016model,pritzel2017neural};
and finally, (c) assuming fixed combination between episodic and parametric
values \cite{lin2018episodic,hansen2018fast} that makes episodic
contribution weight unchanged for different observations and requires
manual tuning of the weight. We tackle these open issues by designing
a novel model that flexibly integrates habitual, model-based and episodic
control into a single architecture for RL.

To tackle issue (a) the model learns representations of the trajectory
by minimizing a self-supervised loss. The loss encourages reconstruction
of past observations, thus enforcing a compressive and noise-tolerant
representation of the trajectory for the episodic memory. Unlike model-based
RL \cite{sutton1991dyna,ha2018recurrent} that simulates the world,
our model merely captures the trajectories.

To address issue (b), we propose a model-based value estimation mechanism
established on the trajectory representations. This allows us to design
a memory-based planning algorithm, namely Model-based Episodic Control
(MBEC), to compute the action value online at the time of making decisions.
Hence, our memory does not need to store actions. Instead, the memory
stores trajectory vectors as the keys, each of which is tied to a
value, facilitating nearest neighbor memory lookups to retrieve the
value of an arbitrary trajectory (memory $\mathrm{\memread}$). To
hasten value propagation and reduce noise inside the memory, we propose
using a weighted averaging $\mathrm{write}$ operator that writes
to multiple memory slots, plus a bootstrapped $\mathrm{refine}$ operator
to update the written values at any step.

Finally, to address issue (c), we create a flexible CLS architecture,
merging complementary systems of learning and memory. An episodic
value is combined with a parametric value via dynamic consolidation.
Concretely, conditioned on the current observation, a neural network
dynamically assigns the combination weight determining how much the
episodic memory contributes to the final action value. We choose DQN
as the parametric value function and train it to minimize the temporal
difference (TD) error (habitual control). The learning of DQN takes
episodic values into consideration, facilitating a distillation of
the episodic memory into the DQN's weights.

Our contributions are: (i) a new model-based control using episodic
memory of trajectories; (ii) a Complementary Learning Systems architecture
that addresses limitations of current episodic RL through a dynamic
hybrid control unifying model-based, episodic and habitual learning
(see Fig. \ref{fig:overview}); and, (iii) demonstration of our architecture
on a diverse test-suite of RL problems from grid-world, classical
control to Atari games and 3D navigation tasks. We show that the MBEC
is noise-tolerant, robust in dynamic grid-world environments. In classical
control, we show the advantages of the hybrid control when the environment
is stochastic, and illustrate how each component plays a crucial role.
For high-dimensional problems, our model also achieves superior performance.
Further, we interpret model behavior and provide analytical studies
to validate our empirical results.

\section{Background }

\subsection{Deep Reinforcement Learning}

Reinforcement learning aims to find the policy that maximizes the
future cumulative rewards of sequential decision-making problems \cite{sutton2018reinforcement}.
Model-based approaches build a model of how the environment operates,
from which the optimal policy is found through planning \cite{sutton1991dyna}.
Recent model-based RL methods can simulate complex environments enabling
sample-efficiency through allowing agents to learn within the simulated
``worlds'' \cite{hafner2019learning,kaiser2019model,hafner2020mastering}.
Unlike these works, Q-learning \cite{watkins1992q} -- a typical
model-free method, directly estimates the true state-action value
function. The function is defined as $Q\left(s,a\right)=\mathbb{E}_{\pi}\left[\sum_{t}\gamma^{t}r_{t}\mid s,a\right]$,
where $r_{t}$ is the reward at timestep $t$ that the agent receives
from the current state $s$ by taking action $a$, followed policy
$\pi$. $\gamma\in\left(0,1\right]$ is the discount factor that weights
the importance of upcoming rewards. Upon learning the function, the
best action can be found as $a_{t}=\underset{a}{\argmax}\,Q\left(s_{t},a\right)$.

With the rise of deep learning, neural networks have been widely used
to improve reinforcement learning. Deep Q-Network (DQN; \cite{mnih2015human})
learns the value function $Q_{\theta}\left(s,a\right)$ using convolutional
and feed-forward neural networks whose parameters are $\theta$. The
value network takes an image representation of the state $s_{t}$
and outputs a vector containing the value of each action $a_{t}$.
To train the networks, DQN samples observed transition $\left(s_{t},a_{t},r_{t},s_{t+1}\right)$
from a replay buffer to minimize the squared error between the value
output and target $y_{t}=r_{t}+\gamma\underset{a}{\max}\,Q_{\theta}^{\prime}\left(s_{t+1},a\right)$
where $Q_{\theta}^{\prime}$ is the target network. The parameter
of the target network is periodically set to that of the value network,
ensuring stable learning. The value network of DQN resembles a semantic
memory that gradually encodes the value of state-action pairs via
replaying as a memory consolidation in CLS theory \cite{kumaran2016learning}. 

Experience replay is critical for DQN, yet it is slow, requiring a
lot of observations since the replay buffer only stores raw individual
experiences. Prioritized Replay \cite{schaul2015prioritized} improves
replaying process with non-uniform sampling favoring important transitions.
Others overcome the limitation of one-step transition by involving
multi-step return in calculating the value \cite{lee2019sample,he2019learning}.
These works require raw trajectories and parallel those using episodic
memory that persists across episodes.

\subsection{Memory-based Controls}

Episodic control enables sample-efficiency through explicitly storing
the association between returns and state-action pairs in episodic
memory \cite{lengyel2008hippocampal,blundell2016model,pritzel2017neural}.
When combined with Q-learning (habitual control), the episodic memory
augments the value function with episodic value estimation, which
is shown beneficial to guide the RL agent to latch on good policies
during early training \cite{lin2018episodic,zhu2019episodic,hu2021generalizable}. 

In addition to episodic memory, Long Short-Term Memory (LSTM; \cite{hochreiter1997long})
and Memory-Augmented Neural Networks (MANNs; \cite{graves2014neural,graves2016hybrid})
are other forms of memories that are excel at learning long sequences,
and thus extend the capability of neural networks in RL. In particular,
Deep Recurrent Q-Network (DRQN; \cite{hausknecht2015deep}) replaces
the feed-forward value network with LSTM counterparts, aiming to solve
Partially-Observable Markov Decision Process (POMDP). Policy gradient
agents are commonly equipped with memories \cite{mnih2016asynchronous,graves2016hybrid,schulman2017proximal}.
They capture long-term dependencies across states, enrich state representation
and contribute to making decisions that require reference to past
events in the same episode. 

Recent works use memory for reward shaping either via contribution
analysis \cite{arjona2019rudder} or memory attention \cite{hung2019optimizing}.
To improve the representation stored in memory, some also construct
a model of transitions using unsupervised learning \cite{wayne2018unsupervised,ha2018recurrent,fortunato2019generalization}.
As these memories are cleared at the end of the episode, they act
more like working memory with a limited lifespan \cite{baddeley1974working}.
Relational \cite{zambaldi2018deep,pmlr-v119-le20b} and Program Memory
\cite{le2020neurocoder} are other forms of memories that have been
used for RL. They are all different from the notion of persistent
episodic memory, which is the main focus of this paper.

\section{Methods}

\begin{figure*}
\begin{centering}
\includegraphics[width=1\linewidth]{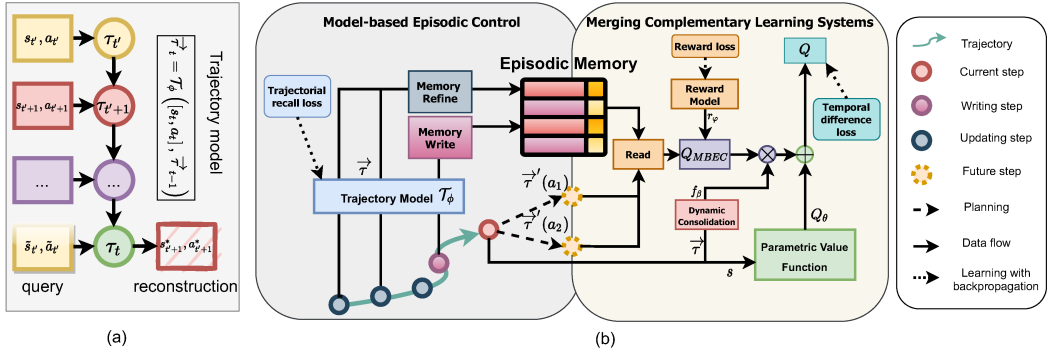}
\par\end{centering}
\caption{(a) Trajectorial Recall. The trajectory model reconstructs any past
observation along the trajectory given noisy preceding s-a pair as
a cue. (b) Dynamic hybrid control with the episodic memory at its
core. The trajectory model, trained with TR loss (Eq. \ref{eq:trj_l}),
encodes representations for writing (Eq. \ref{eq:knnw}) and updating
(Eq. \ref{eq:up}) the episodic memory. Model-based Episodic Control
(MBEC) plans actions (e.g. $a_{1}$ and $a_{2}$) and computes future
trajectory representations ($\protect\overrightarrow{\tau}^{\prime}\left(a_{1}\right)$
and $\protect\overrightarrow{\tau}^{\prime}\left(a_{2}\right)$) for
reading the memory's stored values. The read-out, together with the
reward model, estimates the episodic value $Q_{MBEC}$ (Eq. \ref{eq:mbec}).
The Complementary Learning Systems (CLS) combines $Q_{MBEC}$ and
the traditional semantic $Q_{\theta}$ using dynamic consolidation
conditioned on $\protect\overrightarrow{\tau}$ (Eq. \ref{eq:qe-qs}).
The parameters of the CLS are optimized with TD loss (Eq. \ref{eq:td_l};
habitual control). \label{fig:overview}}
\end{figure*}

We introduce a novel model that combines habitual, model based and
episodic control wherein episodic memory plays a central role. Thanks
to the memory storing trajectory representations, we can estimate
the value in a model-driven fashion: for any action considered, the
future trajectory is computed to query the episodic memory and get
the action value. This takes advantage of model-based planning and
episodic inference. The episodic value is then fused with a slower,
parametric value function to leverage the capability of both episodic
and habitual learning. Fig. \ref{fig:overview} illustrates these
components. We first present the formation of the trajectory representation
and the episodic memory. Next, we describe how we estimate the value
from this memory and parametric Q networks.

\subsection{Episodic Memory Stores Trajectory Representations}

In this paper, a trajectory $\tau_{t}$ is a sequence of what happens
up to the time step $t$: $\tau_{t}=\left[(s_{1},a_{1}),...,(s_{t},a_{t})\right]$.
If we consider $\tau_{t}$ as a piece of memory that encodes events
in an episode, from that memory, we must be able to recall any past
event. This ability in humans can be examined through the serial recall
test wherein a previous item cues to the recall of the next item in
a sequence \cite{farrell2012temporal}. We aim to represent trajectories
as vectors manifesting that property. As such, we employ a recurrent
trajectory network $\mathcal{T}_{\phi}$ to model $\overrightarrow{\tau}_{t}=\mathcal{T_{\phi}}\left(\left[s_{t},a_{t}\right],\overrightarrow{\tau}_{t-1}\right)$
where $\mathcal{T}_{\phi}$ is implemented as an LSTM \cite{hochreiter1997long}
and $\overrightarrow{\tau}_{t}\in\mathbb{R}^{H}$ is the vector representation
of $\tau_{t}$ and also the hidden state of the LSTM.

We train the \emph{trajectory model} $\mathcal{T}_{\phi}$ to compute
$\overrightarrow{\tau}_{t}$ such that it is able to reconstruct the
\emph{next} observation of any \emph{past} experience, simulating
the serial recall test. Concretely, given a noisy version $\left(\tilde{s}_{t^{\prime}},\tilde{a}_{t^{\prime}}\right)$
of a query state-action $\left(s_{t^{\prime}},a_{t^{\prime}}\right)$
sampled from a trajectory buffer $\mathcal{B}$ at time step $t^{\prime}<t$,
we minimize the\emph{ trajectorial recall (TR) loss} as follows,

\begin{align}
\mathcal{L}_{tr} & =\mathbb{E}\left(\left\Vert y^{*}\left(t\right)-\left[s_{t^{\prime}+1},a_{t^{\prime}+1}\right]\right\Vert _{2}^{2}\right)\label{eq:trj_l}
\end{align}
where $y^{*}\left(t\right)=\mathcal{G}_{\omega}\left(\mathcal{T_{\phi}}\left(\left[\tilde{s}_{t^{\prime}},\tilde{a}_{t^{\prime}}\right],\overrightarrow{\tau}_{t}\right)\right)$,
$\mathcal{G}_{\omega}$ is a reconstruction function, implemented
as a feed-forward neural network. The trajectory network $\mathcal{T}_{\phi}$
must implement some form of associative memory, compressing information
of any state-action query in the past in its current representation
$\overrightarrow{\tau}_{t}$ to reconstruct the next observation of
the query, keeping the TR loss low (see a visualization in Fig. \ref{fig:overview}
(a)). Appendix \ref{subsec:Why-Is-Trajectorial} theoretically explains
why the TR loss is suitable for episodic control. 

Our goal is to keep the trajectory representation and its associated
value as the key and value of an episodic memory $\mathcal{M}=\left\{ \mathcal{M}^{k},\mathcal{M}^{v}\right\} $,
respectively. $\mathcal{M}^{k}\in\mathbb{R}^{N\times H}$ and $\mathcal{M}^{v}\in\mathbb{R}^{N\times1}$,
where $N$ is the maximum number of memory slots. The true value of
a trajectory $\overrightarrow{\tau}_{t}$ is simply defined as the
value of the resulting state of the trajectory $V\left(\overrightarrow{\tau}_{t}\right)=V\left(s_{t+1}\right)=\mathbb{E}\left(\sum_{i=0}^{T-t-1}\gamma^{i}r_{t+1+i}\right)$,
$T$ is the terminal step. The memory stores estimation of the true
trajectory values through averaging empirical returns by our weighted
average writing mechanism (see in the next section). 

\begin{algorithm}[t]
\begin{algorithmic}[1]
\State{$\mathcal{D}$: Replay buffer; $\mathcal{B}$: Trajectory buffer; $\mathcal{C}$: Chunk buffer; $\mathcal{M}$: Episodic memory; $\mathcal{T}_{\phi}$: Trajectory model; $r_{\varphi}\left(s,a\right)$: Reward model; $L$: Chunk length}
\ForEach{episode}
\State{Initialize $\overrightarrow{\tau}_{0}=0$; $\mathcal{B}=\emptyset$; $\mathcal{C}=\emptyset$}
\For{$t=1,T$}
\State{Observe state $s_t$. Compute $f_{\beta}\left(\overrightarrow{\tau}_{t-1}\right)$.} 
\State{Select action $a_t \leftarrow \epsilon$-greedy policy using Q in Eq. $\ref{eq:qe-qs}$}
\State{Observe reward $r_t$. Move to next state $s_{t+1}$. Compute $\overrightarrow{\tau}_{t} \leftarrow \mathcal{T_{\phi}}\left(\left[s_{t},a_{t}\right],\overrightarrow{\tau}_{t-1}\right)$}
\State{Add ($s_t,a_t,s_{t+1},r_t,\overrightarrow{\tau}_{t-1},\overrightarrow{\tau}_{t}$) and ($s_t,a_t$) to $\mathcal{D}$ and  $\mathcal{B}$, respectively.}
\State{Refine memory: $\mathcal{M} \leftarrow \mathrm{refine}\left(s_{t},\overrightarrow{\tau}_{t-1}|\mathcal{M}\right)$}
\For{($s_{t^{\prime}},a_{t^{\prime}},s_{t^{\prime}+1},r_{t^{\prime}},\overrightarrow{\tau}_{t^{\prime}-1},\overrightarrow{\tau}_{t^{\prime}}$) sampled from $\mathcal{D}$}

\State{Compute $Q\left(s,a\right)$ using $s_{t^{\prime}},a_{t^{\prime}},\overrightarrow{\tau}_{t^{\prime}-1}$ (Eq. $\ref{eq:qe-qs}$)}
\State{Compute $Q\left(s^{\prime},a^{\prime}\right)$ using $s_{t^{\prime}+1},\overrightarrow{\tau}_{t^{\prime}}$ $\forall a^{\prime}$  (Eq. $\ref{eq:qe-qs}$)}
\State{Optimize $\mathcal{L}_{q}$ wrt $\theta$  and $\beta$  (Eq. $\ref{eq:td_l}$)}

\State{Compute $\mathcal{L}_{re}$ using $s_{t^{\prime}},a_{t^{\prime}},r_{t^{\prime}}$ (Eq.  $\ref{eq:reward_l}$).  Optimize $\mathcal{L}_{re}$ wrt $\varphi$}
\EndFor

\If{$t \bmod L==0$}
\State{Add ($\overrightarrow{\tau}_{t-1}$, $r_t$) to $\mathcal{C}$. Sample ($s_{t^{\prime}},a_{t^{\prime}}$) from $\mathcal{B}$  and optimize $\mathcal{L}_{tr}$ wrt $\phi$ and  $\omega$  (Eq. $\ref{eq:trj_l})$}
\EndIf

\If{$t==T$}
\ForEach {$\overrightarrow{\tau}_{i} \in \mathcal{C}$}
\State{Compute $\hat{V}\left(\overrightarrow{\tau}_{i}\right)=\sum_{j=i}^{T-1}\gamma^{j-i}r_{j+1}$. Write $\mathcal{M}\leftarrow \mathrm{\memwrite}\left(\overrightarrow{\tau}_{i},\hat{V}\left(\overrightarrow{\tau}_{i}\right)|\mathcal{M}\right)$}
\EndFor
\EndIf
\EndFor
\EndFor
\end{algorithmic}

\caption{MBEC++: Complementary reinforcement learning with MBEC and DQN.\label{alg:Model-based-Episodic-Control}}
\end{algorithm}

\subsection{Memory Operators}

\textbf{Memory reading} Given a query key $\overrightarrow{\tau}$,
we read from the memory the corresponding value by referring to neighboring
representations. Concretely, two reading rules are employed 

\begin{eqnarray*}
\mathrm{\memread}\left(\overrightarrow{\tau}|\mathcal{M}\right) & = & \begin{cases}
\sum_{i\in\mathcal{N}^{K_{r}}(\overrightarrow{\tau})}\frac{\left\langle \mathcal{M}_{i}^{k},\overrightarrow{\tau}\right\rangle \mathcal{M}_{i}^{v}}{\sum_{j\in\mathcal{N}^{K}(\overrightarrow{\tau})}\left\langle \mathcal{M}_{j}^{k},\overrightarrow{\tau}\right\rangle } & \left(a\right)\\
\max_{i\in\mathcal{N}^{K_{r}}(\overrightarrow{\tau})}\mathcal{M}_{i}^{v} & \left(b\right)
\end{cases}
\end{eqnarray*}
where $\left\langle \cdot\right\rangle $ is a kernel function and
$\mathcal{N}^{K_{r}}\left(\cdot\right)$ retrieves top $K_{r}$ nearest
neighbors of the query in $\mathcal{M}^{k}$. The read-out is an estimation
of the value of the trajectory $\overrightarrow{\tau}$ wherein the
weighted average rule $\left(a\right)$ is a balanced estimation,
while the max rule $\left(b\right)$ is optimistic, encouraging exploitation
of the best local experience. In this paper, the two reading rules
are simply selected randomly with a probability of $p_{read}$ and
$1-p_{read}$, respectively. 

\textbf{Memory writing} Given the writing key $\overrightarrow{\tau}$
and its estimated value $\hat{V}\left(\overrightarrow{\tau}\right)$,
the write operator $\mathrm{\memwrite}\left(\overrightarrow{\tau},\hat{V}\left(\overrightarrow{\tau}\right)|\mathcal{M}\right)$
consists of several steps. First, we add the value to the memories
$\mathcal{M}^{v}$ if the key cannot be found in the key memory $\mathcal{M}^{k}$
(this happens frequently since key match is rare). Then, we update
the values of the key neighbors such that the updated values are approaching
the written value $\hat{V}\left(\overrightarrow{\tau}\right)$ with
speeds relative to the distances as $\forall i\in\mathcal{N}^{K_{w}}(\overrightarrow{\tau}):$

\begin{align}
\mathcal{M}_{i}^{v} & \leftarrow\mathcal{M}_{i}^{v}+\alpha_{w}\left(\hat{V}\left(\overrightarrow{\tau}\right)-\mathcal{M}_{i}^{v}\right)\frac{\left\langle \mathcal{M}_{i}^{k},\overrightarrow{\tau}\right\rangle }{\sum_{j\in\mathcal{N}^{K_{w}}(\overrightarrow{\tau})}\left\langle \mathcal{M}_{j}^{k},\overrightarrow{\tau}\right\rangle }\label{eq:knnw}
\end{align}
where $\alpha_{w}$ is the writing rate. Finally, the key can be added
to the key memory. When it exceeds memory capacity $N$, the earliest
added will be removed. For simplicity, $K_{w}=K_{r}=K$.

We note that our memory writing allows multiple updates to multiple
neighbor slots, which is unlike the single-slot update rule \cite{blundell2016model,pritzel2017neural,lin2018episodic}.
Here, the written value is the Monte Carlo return collected from $t+1$
to the end of the episode $\hat{V}\left(\overrightarrow{\tau}_{t}\right)=\sum_{i=0}^{T-t-1}\gamma^{i}r_{t+1+i}$.
Following \cite{le2019learning}, we choose to write the trajectory
representation of every $L$-th time-step (rather than every time-step)
to save computation while still maintaining good memorization. Appendix
\ref{subsec:Convergence-Analysis-of} provides a mean convergence
analysis of our writing mechanism. 

\textbf{Memory refining} As the memory writing is only executed after
the episode ends, it delays the value propagation inside the memory.
Hence, we design the $\mathrm{refine}\left(s_{t},\overrightarrow{\tau}_{t-1}|\mathcal{M}\right)$
operator that tries to minimize the one-step TD error of the memory's
value estimation. As such, at an arbitrary timestep $t$, we estimate
the future trajectory if the agent takes action $a$ using the trajectory
model as $\overrightarrow{\tau}_{t}^{\prime}\left(a\right)=\mathcal{T}_{\phi}\left(\left[s_{t},a\right],\overrightarrow{\tau}_{t-1}\right)$.
Then, we can update the memory values as follows,

\begin{minipage}[t]{0.49\columnwidth}%
\begin{equation}
Q'=\underset{a}{\max}\,r_{\varphi}\left(s_{t},a\right)+\gamma\memread\left(\overrightarrow{\tau}_{t}^{\prime}\left(a\right)|\mathcal{M}\right)
\end{equation}
\end{minipage}%
\begin{minipage}[t]{0.49\columnwidth}%
\begin{align}
\mathcal{M} & \leftarrow\memwrite\left(\overrightarrow{\tau}_{t-1},Q'|\mathcal{M}\right)\label{eq:up}
\end{align}
\end{minipage}

where $r_{\varphi}$ is a reward model using a feed-forward neural
network. $r_{\varphi}$ is trained by minimizing 

\begin{equation}
\mathcal{L}_{re}=\mathbb{E}\left(r-r_{\varphi}\left(s,a\right)\right)^{2}\label{eq:reward_l}
\end{equation}
 The memory refining process can be shown to converge in finite MDP
environments.
\begin{prop}
In a finite MDP ($\mathcal{S}$,$\mathcal{A}$,$\mathcal{T}$,$\mathcal{R}$),
given a fixed bounded $r_{\varphi}$ and an episodic memory $\mathcal{M}$
with $\mathrm{\memread}$ (average rule) and $\mathrm{\memwrite}$
operations, the memory $\mathrm{refine}$ given by Eq. \ref{eq:up}
converges to a fixed point with probability 1 as long as $\gamma<1$,
$\sum_{t=1}^{\infty}\alpha_{w,t}=\infty$ and $\sum_{t=1}^{\infty}\alpha_{w,t}^{2}<\infty$.
\end{prop}

\begin{proof}
See Appendix \ref{subsec:Convergence-Analysis-of-1}.
\end{proof}

\subsection{Model-based Episodic Control (MBEC)}

Our agent relies on the memory at every timestep to choose its action
for the current state $s_{t}$. To this end, the agent first plans
some action and uses $\mathcal{T_{\phi}}$ to estimate the future
trajectory. After that, it reads the memory to get the value of the
planned trajectory. This mechanism takes advantage of model-based
RL's planning and episodic control's non-parametric inference, yielding
a novel hybrid control named Model-based Episodic Control (MBEC).
The state-action value then is

\begin{equation}
Q_{MBEC}\left(s,a\right)=r_{\varphi}\left(s,a\right)+\gamma\memread\left(\overrightarrow{\tau}^{\prime}\left(a\right)|\mathcal{M}\right)\label{eq:mbec}
\end{equation}
The MBEC policy then is $\pi\left(s\right)=\underset{a}{\argmax}\:Q_{MBEC}\left(s,a\right)$.
Unlike model-free episodic control, we compute on-the-fly instead
of storing the state-action value. Hence, the memory does not need
to store all actions to get reliable action values. 

\begin{figure*}
\begin{centering}
\includegraphics[width=0.95\linewidth]{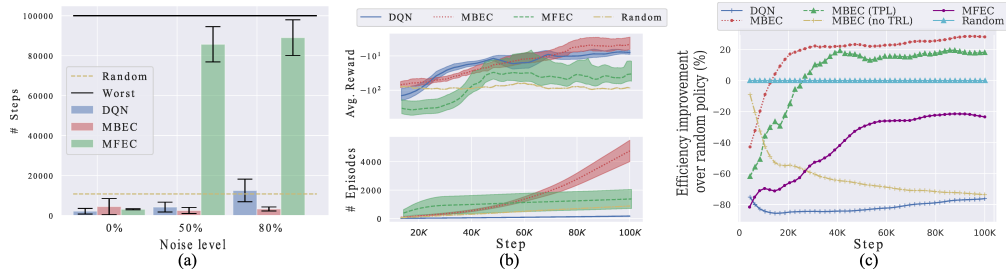}
\par\end{centering}
\caption{Maze navigation. (a) Noisy mode: number of steps required to complete
100 episodes with different noise rates (lower is better). (b) Trap
mode: average reward (upper) and number of completed episodes (lower)
over timesteps (higher is better). (c) Dynamic mode: efficiency improvement
over random policy across timesteps (higher is better).\label{fig:maze}}
\end{figure*}

\subsection{Model-based Episodic Control Facilitates Complementary Learning Systems}

The episodic value provides direct yet biased estimation from experiences.
To compensate for that, we can use a neural network $Q_{\theta}\left(s_{t},a_{t}\right)$
to give an unbiased value estimation \cite{mnih2015human}, representing
the slow-learning semantic memory that gradually encodes optimal values.
Prior works combine by a weighted summation of the episodic and semantic
value wherein the weight is fixed \cite{lin2018episodic,hansen2018fast}.
We believe that depending on the specific observations, we may need
different weights to combine the two values. Hence, we propose to
combine the two episodic and semantic systems as

\begin{equation}
Q\left(s_{t},a_{t}\right)=Q_{MBEC}\left(s_{t},a_{t}\right)f_{\beta}\left(\overrightarrow{\tau}_{t-1}\right)+Q_{\theta}\left(s_{t},a_{t}\right)\label{eq:qe-qs}
\end{equation}
where $f_{\beta}$ is a feed-forward neural network with sigmoid activation
that takes the previous trajectory as the input and outputs a consolidating
weight for the episodic value integration. This allows dynamic integration
conditioned on the trajectory status. The semantic system learns to
take episodic estimation into account in making decisions via replaying
to minimize one-step TD error,

\begin{equation}
\mathcal{L}_{q}=\mathbb{E}\left(r+\gamma\max_{a^{\prime}}Q\left(s^{\prime},a^{\prime}\right)-Q\left(s,a\right)\right)^{2}\label{eq:td_l}
\end{equation}
Here we note that $Q_{MBEC}$ is also embedded in the target, providing
better target value estimation in the early phase of learning when
the parametric model does not learn well. We follow \cite{mnih2015human}
using a replay buffer $\mathcal{D}$ to store $\left(s,a,s^{\prime},r\right)$
across episodes. Without episodic contribution, TD or habitual learning
is slow \cite{hansen2018fast,lin2018episodic}. Our episodic integration
allows the agent to rely on MBEC whenever it needs to compensate for
immature parametric systems. Alg. \ref{alg:Model-based-Episodic-Control}
(MBEC++) summarizes MBEC operations within the complementary learning
system. The two components (MBEC and CLS) are linked by the episodic
memory as illustrated in Fig. \ref{fig:overview} (b).

\section{Results}

\begin{figure*}
\begin{centering}
\includegraphics[width=0.95\textwidth]{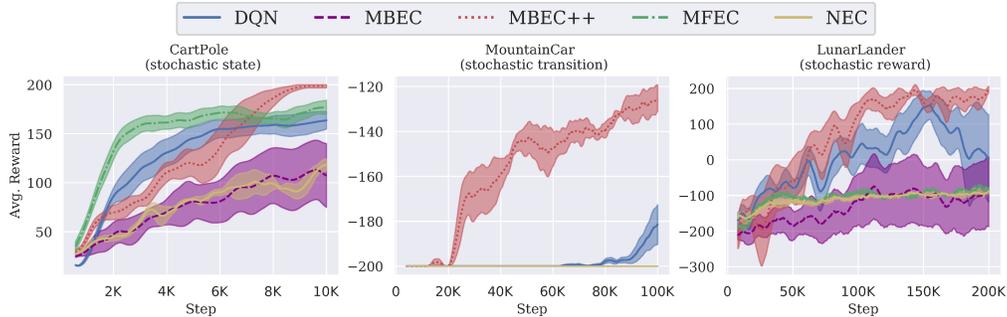}
\par\end{centering}
\caption{Average reward over learning steps on representative stochastic classical
control environments (higher is better, mean and std. over 10 runs).\label{fig:sto_class3}}
\end{figure*}

In this section, we examine our proposed episodic control both as
a stand-alone (MBEC) and within a complementary learning system (MBEC++).
To keep our trajectory model simple, for problems with image state,
it learns to reconstruct the feature vector of the image input, rather
than the image itself. The main baselines are DQN \cite{mnih2015human}
and recent (hybrid) episodic control methods. Details of the baseline
configurations and hyper-parameter tuning for each tasks can be found
in Appendix \ref{subsec:Experimental-Details}. 

\subsection{Grid-world: 2D Maze Exploration\label{subsec:Gridworld:-2D-Maze}}

We begin with simple maze navigation to explore scenarios wherein
our proposed memory shows advantages. In this task, an agent is required
to move from the starting point $\left(0,0\right)$ to the endpoint
$\left(n_{e}-1,n_{e}-1\right)$ in a maze environment of size $n_{e}\times n_{e}$.
In the maze task, if the agent hits the wall of the maze, it gets
$-1$ reward. If it reaches the goal, it gets $+1$ reward. For each
step in the maze, the agent get $-0.1/n_{e}^{2}$ reward. An episode
ends either when the agent reaches the goal or the number of steps
exceeds 1000. We create different scenarios for this task ($n_{e}=3$):
noisy, trap and dynamic modes wherein our MBEC is compared with both
parametric (DQN) and memory-based (MFEC; \cite{blundell2016model})
models (see Appendix \ref{subsec:Maze-task} for details and more
results).

\textbf{Noisy mode} In this mode, the state is represented as the
location plus an image of the maze. The image is preprocessed by a
pretrained ResNet, resulting in a feature vector of $512$ dimensions
(the output layer before softmax). We apply dropout to the image vector
with different noise levels. We hypothesize that aggregating states
into trajectory representation as in MBEC is a natural way to smooth
out the noise of individual states. 

Fig. \ref{fig:maze} (a) measures sample efficiency of the models
on noisy mode. Without noise, all models can quickly learn the optimal
policy and finish 100 episodes within  1000 environment steps. However,
as the increased noise distracts the agents, MFEC cannot find a way
out until the episode ends. DQN performance reduces as the noise increases
and ends up even worse than random exploration. By contrast, MBEC
tolerates noise and performs much better than DQN and the random agent.

\textbf{Trap mode} The state is the position of the agent plus a trap
location randomly determined at the beginning of the episode. If the
agent falls into the trap, it receives a $-2$ reward (the episode
does not terminate). This setting illustrates the advantage of memory-based
planning. With MBEC, the agent remembers to avoid the trap by examining
the future trajectory to see if it includes the trap. Estimating state-action
value (DQN and MFEC) introduces overhead as they must remember all
state-action pairs triggering the trap. 

We plot the average reward and number of completed episodes over time
in Fig. \ref{fig:maze} (b). In this mode, DQN always learns a sub-optimal
policy, which is staying in the maze. It avoids hitting the wall and
trap, however, completes a relatively low number of episodes. MFEC
initially learns well, quickly completing episodes. However its learning
becomes unstable as more observations are written into its memory,
leading to a lower performance over time. MBEC alone demonstrates
stable learning, significantly outperforming other baselines in both
reward and sample efficiency. 

\textbf{Dynamic mode} The state is represented as an image and the
maze structure randomly changes for each episode. A learnable CNN
is equipped for each baseline to encode the image into a $64$-dimensional
vector. In this dynamic environment, similar state-actions in different
episodes can lead to totally different outcomes, thereby highlighting
the importance of trajectory modeling. In this case, MFEC uses VAE-CNN,
trained to reconstruct the image state. Also, to verify the contribution
of TR loss, we add two baselines: (i) MBEC without training trajectory
model (no TRL) and (ii) MBEC with a traditional model-based transition
prediction loss (TPL) (see Appendix \ref{subsec:Maze-task} for more
details).

We compare the models' efficiency improvement over random policy by
plotting the percentage of difference between the models' number of
finished episodes and that of random policy in Fig. \ref{fig:maze}
(c). DQN and MFEC perform worse than random. MBEC with untrained trajectory
model performs poorly. MBEC with trajectory model trained with TPL
shows better performance than random, yet still underperforms our
proposed MBEC with TRL by around 5-10\%. 

\begin{figure}
\begin{centering}
\includegraphics[width=0.95\columnwidth]{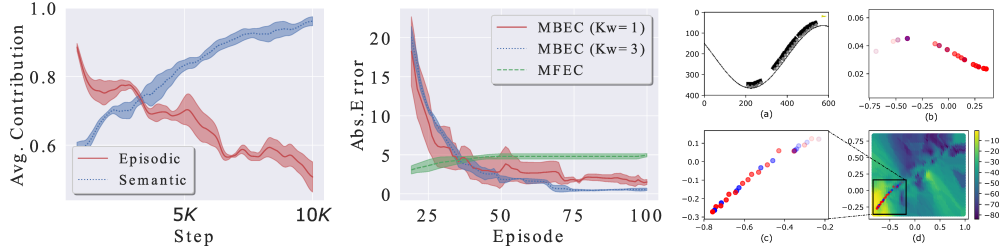}
\par\end{centering}
\caption{Cart Pole. (Left) Average contribution of episodic and semantic value
estimations over timesteps (see Appendix \ref{subsec:app_scc} for
contribution definition). The episodic influence is gradually replaced
by the semantic's. (Middle) The difference (absolute error) between
the stored and true value of the starting state (mean and std. over
5 runs). \label{fig:mcontrib}(Right) Mountain Car. (a) Visualization
of the car moving uphill over 30 timesteps. Due to noise, the next
state can be observed as the current state with probability $p_{tr}$.
(b) State and (c) trajectory spaces: the axes are the dimension of
the trajectory vectors $\vec{\tau}$ projected into 2d space. Blue
denotes the representation at noisy timestep and red the normal ones.
Fading color denotes earlier timesteps. (d) Value estimation by the
episodic memory for the whole $2$-d trajectory space. \label{fig:vis_mc_full}}
\end{figure}

\subsection{Stochastic Classical Control\label{subsec:Stochastic-Classical-Control}}

In stochastic environments, taking correct actions can still lead
to low rewards due to noisy observations, which negatively affects
the quality of episodic memory. We consider 3 classical problems:
Cart Pole, Mountain Car and Lunar Lander. For each problem, we examine
RL agents in stochastic settings by (i) perturbing the reward with
Gaussian (mean 0, std. $\sigma_{re}=0.2$) or (ii) Bernoulli noise
(with probability $p_{re}=0.2$, the agent receives a reward $-r$
where $r$ is the true reward) and (iii) noisy transition (with probability
$p_{tr}=0.5$, the agent observes the current state as its next state
despite taking any action). In this case, we compare MBEC++ with DQN,
MFEC and NEC \cite{pritzel2017neural}. 

Fig. \ref{fig:sto_class3} shows the performances of the models on
representative environments (full results in Appendix \ref{subsec:app_scc}).
For easy problems like Cart Pole, although MFEC learns quickly, its
over-optimistic control is sub-optimal in non-deterministic environments,
and thus cannot solve the task. For harder problems, stochasticity
makes state-based episodic memories (MFEC, NEC) fail to learn. DQN
learns in these settings, albeit slowly, illustrating the disadvantage
of not having a reliable episodic memory. Among all, MBEC++ is the
only model that can completely solve the noisy Cart Pole within 10,000
steps and demonstrates superior performance in Mountain Car and Lunar
Lander. Compared to MBEC++, MBEC performs badly, showing the importance
of CLS. 
\begin{flushleft}
\textbf{Behavior analysis} The episodic memory plays a role in MBEC++'s
success. In the beginning, the agent mainly relies on the memory,
yet later, it automatically learns to switch to the semantic value
estimation (see Fig. \ref{fig:mcontrib} (left)). That is because
in the long run the semantic value is more reliable and already fused
with the episodic value through Eq. \ref{eq:qe-qs}-\ref{eq:td_l}.
Our $\memwrite$ operator also helps MBEC++ in quickly searching for
the optimal value. To illustrate that, we track the convergence of
the episodic memory's written values for the starting state of the
stochastic Cart Pole problem under a fixed policy (Fig. \ref{fig:mcontrib}
(middle)). Unlike the over-optimistic MFEC's writing rule using $\max$
operator, ours enables mean convergence to the true value despite
Gaussian reward noise. When using a moderate $K_{w}>1$, the estimated
value converges better as suggested by our analysis in Appendix \ref{subsec:Convergence-Analysis-of}.
Finally, to verify the contribution of the trajectory model, we examine
MBEC++'s ability to counter noisy transition by visualizing the trajectory
spaces for Mountain Car in Fig. \ref{fig:vis_mc_full} (right). Despite
the noisy states (Fig. \ref{fig:vis_mc_full} (right, b)), the trajectory
model can still roughly estimate the trace of trajectories (Fig. \ref{fig:vis_mc_full}
(right, c)). That ensures when the agent approaches the goal, the
trajectory vectors smoothly move to high-value regions in the trajectory
space (Fig. \ref{fig:vis_mc_full} (right, d)). We note that for this
problem that ability is only achieved through training with TR loss
(comparison in Appendix \ref{subsec:app_scc}).
\par\end{flushleft}

\textbf{Ablation study} We pick the noisy transition Mountain Car
problem for ablating components and hyper-parameters of MBEC++ with
different neighbors ($K$), chunk length ($L$) and memory slots ($N$).
The results in Fig. \ref{fig:mc_abl} demonstrate that the performance
improves as $K$ increases, which is common for KNN-based methods.
We also find that using a too short or too long chunk deteriorates
the performance of MBEC++. Short chunk length creates redundancy as
the stored trajectories will be similar while long chunk length makes
minimizing TR loss harder. Finally, the results confirm that the learning
of MBEC++ is hindered significantly with small memory. A too-big memory
does not help either since the trajectory model continually refines
the trajectory representation, a too big memory slows the replacement
of old representations with more accurate newer ones. 

We also ablate MBEC++: (i) without TR loss, (ii) with TP loss (iii),
without multiple write ($K_{w}=1$) and (iv) without memory $\mathrm{update}$.
We realize that the first two configurations show no sign of learning.
The last two can learn but much slower than the full MBEC++, justifying
our neighbor memory writing and update (Fig. \ref{fig:mc_abl} (rightmost)).
More ablation studies are in Appendix \ref{subsec:Ablation-study}
where we find our dynamic consolidation is better than fixed combinations
and optimal $p_{read}$ is 0.7.

\begin{table}
\centering{}%
\begin{tabular}{ccc}
\hline 
\multirow{1}{*}{{\small{}Model }} & \multirow{1}{*}{{\small{}All}} & {\small{}25 games}\tabularnewline
\hline 
{\small{}Nature DQN} & {\small{}15.7/51.3} & {\small{}83.6/16.0}\tabularnewline
{\small{}MFEC} & {\small{}85.0/45.4} & {\small{}77.7/40.9}\tabularnewline
{\small{}NEC} & {\small{}99.8/54.6} & {\small{}106.1/53.3}\tabularnewline
{\small{}EMDQN{*}} & {\small{}528.4/92.8} & {\small{}250.6/95.5}\tabularnewline
{\small{}EVA} & {\small{}-} & {\small{}172.2/39.2}\tabularnewline
{\small{}ERLAM} & {\small{}-} & {\small{}515.4/103.5}\tabularnewline
\hline 
{\small{}MBEC++} & \textbf{\small{}654.0/117.2} & \textbf{\small{}518.2/133.4}\tabularnewline
\hline 
\end{tabular}\caption{Human normalized scores (mean/median) at 10 million frames for all
and a subset of 25 games. Baselines' numbers are adopted from original
papers and \cite{zhu2019episodic}, respectively. {*} The baseline
is reported with 40 million frames of training. - The exact numbers
are not reported.\label{tab:Human-noramlzed-scores}}
\end{table}

\begin{figure}
\begin{centering}
\includegraphics[width=1\textwidth]{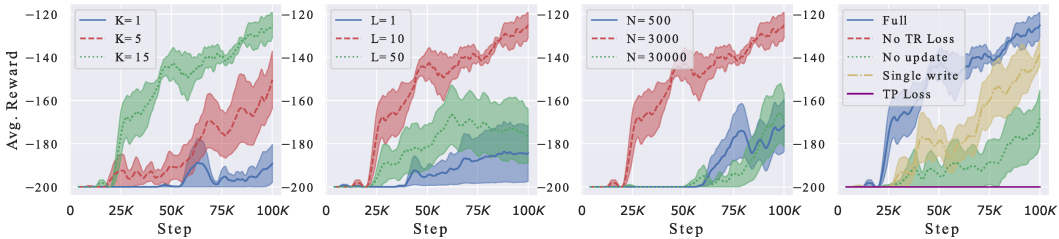}
\par\end{centering}
\caption{Noisy Transition Mountain Car: ablation study. Each plot varies one
hyper-parameter or ablated component while fixing others in default
values ($K=15$, $L=10$, $N=3000$).\label{fig:mc_abl}}
\end{figure}

\subsection{Atari 2600 Benchmark\label{subsec:Atari-2600-Benchmark}}

We benchmark MBEC++ against other strong episodic memory models in
playing Atari 2600 video games \cite{bellemare2013arcade}. The task
can be challenging with stochastic and partially observable games
\cite{kaiser2019model}. Our model adopts DQN \cite{mnih2015human}
with the same setting (details in Appendix \ref{subsec:app_atari}).
We only train the models within 10 million frames for sample efficiency.

Table \ref{tab:Human-noramlzed-scores} reports the average performance
of MBEC++ and baselines on all (57) and 25 popular Atari games concerning
human normalized metrics \cite{mnih2015human}. Compared to the vanilla
DQN, MFEC and NEC, MBEC++ is significantly faster at achieving high
scores in the early learning phase. Further, MBEC++ outperforms EMDQN
even trained with 40 million frames and achieves the highest median
score. Here, state-action value estimations fall short in quickly
solving complicated tasks with many actions like playing Atari games
as it takes time to visit enough state-action pairs to create useful
memory's stored values. By contrast, when the models in MBEC++ are
well-learnt (which is often within 5 million frames, see Appendix
\ref{subsec:app_atari}), its memory starts providing reliable trajectory
value estimation to guide the agent to good policies. Remarkably,
our episodic memory is much smaller than that of others and our trajectory
model size is insignificant to DQN's (Appendix \ref{subsec:app_atari}
and Table \ref{tab:Number-of-trainable}). 

In the subset testbed, MBEC++ demonstrates competitive performance
against trajectory-utilized models including EVA \cite{hansen2018fast}
and ERLAM \cite{zhu2019episodic}. These baselines work with trajectories
as raw state-action sequences, unlike our distributed trajectories.
In the mean score metric, MBEC++ is much better than EVA (nearly double
score) and slightly better than ERLAM. MBEC++ agent plays consistently
well across games without severe fluctuations in performance, indicated
by its significantly higher median score. 

We also compare MBEC++ with recent model-based RL approaches including
Dreamer-v2 \cite{hafner2020mastering} and SIMPLE \cite{kaiser2019model}.
The results show that our method is competitive against these baselines.
Notably, our trajectory model is much simpler than the other methods
(we only have TR and reward losses and our network is the standard
CNN of DQN for Atari games). Appendix \ref{subsec:app_atari} provides
more details, learning curves and further analyses.

\subsection{POMDP: 3D Navigation\label{subsec:POMDP:-3D-Navigation}}

To examine MBEC++ on Partially-Observable Markov Decision Process
(POMDP) environments, we conduct experiments on a 3D navigation task:
Gym Mini-World's Pickup Objects \cite{gym_miniworld}. Here, an agent
moves around a big room to collect several objects ($+1$ reward for
each picked object). The location, shape and color of the objects
change randomly across episodes. The state is the frontal view-port
of the agent and encoded by a common CNN for all baselines (details
in Appendix \ref{subsec:app_3dnav}). 

We train all models for only 2 million steps and report the results
for a different number of objects in Appendix's Fig. \ref{fig:3d}.
Among all baselines, MBEC++ demonstrates the best learning progress
and consistently improves over time. Other methods either fail to
learn (DRQN) or show a much slower learning speed (DQN and PPO). That
proves our MBEC++ is useful for POMDP.

\section{Conclusion}

We have introduced a new episodic memory that significantly accelerates
reinforcement learning in various problems beyond near-deterministic
environments. Its success can be attributed to: (a) storing distributed
trajectories produced by a trajectory model, (b) memory-based planning
with fast value-propagating memory writing and refining, and (c) dynamic
consolidation of episodic values to parametric value function. Our
experiments demonstrate the superiority of our method to prior episodic
controls and strong RL baselines. One limitation of this work is the
large number  of hyperparameters, which prevents us from fully tuning
MBEC++. In future work, we will extend to continuous action space
and explore multi-step memory-based planning capability of our approach. 

Our research aims to improve sample-efficiency of RL and can be trained
with common computers. Our method improves the performance in various
RL tasks, and thus opens the chance for creating better autonomous
systems that work flexibly across sectors (robotics, manufacturing,
logistics, and decision support systems). Although we do not think
there are immediate bad consequences, we are aware of potential problems.
First, our method does not guarantee safe exploration during training.
If learning happens in a real-world setting (e.g. self-driving car),
the agent can make unsafe exploration (e.g. causing accidents). Second,
we acknowledge that our method, like many other Machine Learning algorithms,
can be misused in unethical or malicious activities.

\subsubsection*{ACKNOWLEDGMENTS}

This research was partially funded by the Australian Government through
the Australian Research Council (ARC). Prof Venkatesh is the recipient
of an ARC Australian Laureate Fellowship (FL170100006).

\bibliographystyle{plain}
\bibliography{dtm}

\newpage{}

\section*{Appendix}

\renewcommand\thesubsection{\Alph{subsection}}

\subsection{Analytical Studies on Model-based Episodic Memory}

\subsubsection{Why Is Trajectorial Recall (TR) Loss Good for Episodic Memory?\label{subsec:Why-Is-Trajectorial}}

For proper episodic control, neighboring keys should represent similar
trajectories. If we simply assume that two trajectories are similar
if they share many common transitions, training the trajectory model
with TR loss indeed somehow enforces that property. To illustrate,
we consider simple linear $\mathcal{T_{\phi}}$ and $\mathcal{G}_{\omega}$
such that the reconstruction process becomes

\[
y^{*}\left(t\right)=W\left(U\overrightarrow{\tau}_{t}+V\left[s_{t^{\prime}},a_{t^{\prime}}\right]\right)
\]
Here, we also assume that the query is clean without added noise.
Then we can rewrite TR loss for a trajectory $\tau_{t}$

\begin{align*}
\mathcal{L}_{tr}\left(\tau_{t}\right) & =\sum_{t^{\prime}=1}^{t}\left\Vert W\left(U\overrightarrow{\tau}_{t}+V\left[s_{t^{\prime}},a_{t^{\prime}}\right]\right)-\left[s_{t^{\prime}+1},a_{t^{\prime}+1}\right]\right\Vert _{2}^{2}\\
 & =\sum_{t^{\prime}=1}^{t}\left\Vert \Delta_{t^{\prime}}\left(\tau_{t}\right)\right\Vert _{2}^{2}
\end{align*}

Let us denote $S\neq\emptyset$ the set of common transition steps
between 2 trajectories: $\tau_{t_{1}}^{1}$ and $\tau_{t_{2}}^{2}$,
by applying triangle inequality,

\begin{align*}
\mathcal{L}_{tr}\left(\tau_{t_{1}}^{1}\right)+\mathcal{L}_{tr}\left(\tau_{t_{2}}^{2}\right) & \geq\sum_{t^{\prime}\in S}\left\Vert \Delta_{t^{\prime}}\left(\tau_{t_{1}}^{1}\right)\right\Vert _{2}^{2}+\left\Vert \Delta_{t^{\prime}}\left(\tau_{t_{2}}^{2}\right)\right\Vert _{2}^{2}\\
 & \geq\sum_{t^{\prime}\in S}\left\Vert \Delta_{t^{\prime}}\left(\tau_{t_{1}}^{1}\right)-\Delta\left(\tau_{t_{2}}^{2}\right)\right\Vert _{2}^{2}\\
 & =\left|S\right|\left\Vert WU\left(\overrightarrow{\tau}_{t_{1}}^{1}-\overrightarrow{\tau}_{t_{2}}^{2}\right)\right\Vert _{2}^{2}
\end{align*}

If we assume $WU\left(\overrightarrow{\tau}_{t_{1}}^{1}-\overrightarrow{\tau}_{t_{2}}^{2}\right)\neq0$
as $\overrightarrow{\tau}_{t_{1}}^{1}\neq\overrightarrow{\tau}_{t_{2}}^{2}$,
applying Lemma 2.3 in \cite{grcar2010matrix} yields

\[
\frac{\mathcal{L}_{tr}\left(\tau_{t_{1}}^{1}\right)+\mathcal{L}_{tr}\left(\tau_{t_{2}}^{2}\right)}{\left|S\right|\sigma_{min}\left(WU\right)}\geq\left\Vert \overrightarrow{\tau}_{t_{1}}^{1}-\overrightarrow{\tau}_{t_{2}}^{2}\right\Vert _{2}^{2}
\]
where $\sigma_{min}\left(WU\right)$ is the smallest nonzero singular
value of $WU$. As the TR loss decreases and the number of common
transition increases, the upper bound of the distance between two
trajectory vectors decreases, which is desirable. On the other hand,
it is unclear whether the traditional transition prediction loss holds
that property.

\subsubsection{Convergence Analysis of $\protect\memwrite$ Operator\label{subsec:Convergence-Analysis-of}}

In this section, we show that we can always find $\alpha_{w}$ such
that the writing converges with probability $1$ and analyze the convergence
as $\alpha_{w}$ is constant. To simplify the notation, we rewrite
Eq. \ref{eq:knnw} as 

\begin{equation}
\mathcal{M}_{i}^{v}\left(n+1\right)=\mathcal{M}_{i}^{v}\left(n\right)+\text{\ensuremath{\lambda\left(n\right)}}\left(R_{j}\left(n\right)-\mathcal{M}_{i}^{v}\left(n\right)\right)\label{eq:write2}
\end{equation}
where $i$ and $j$ denote the current memory slot being updated and
its neighbor that initiates the writing, respectively. $\lambda\left(n\right)=\alpha_{w}\left(n\right)\frac{\left\langle \right\rangle _{ij}\left(n\right)}{\sum_{b\in\mathcal{N}_{j}^{K_{w}}}\left\langle \right\rangle _{bj}\left(n\right)}$
where $\mathcal{N}_{j}^{K_{w}}$ is the set of $K_{w}$ neighbors
of $j$. $R_{j}$ is the empirical return of the trajectory whose
key is the memory slot $j$, $\left\langle \right\rangle _{ij}$ the
kernel function of 2 keys and $n$ the number of updates. As mentioned
in \cite{sutton2018reinforcement}, this stochastic approximation
converges when $\sum_{n=1}^{\infty}\lambda\left(n\right)=\infty$
and $\sum_{n=1}^{\infty}\lambda^{2}\left(n\right)<\infty$. 

By definition, $\left\langle \right\rangle _{ij}=\frac{1}{\left\Vert \overrightarrow{\tau}_{i}-\overrightarrow{\tau}_{j}\right\Vert +\epsilon}$
and $\left\Vert \overrightarrow{\tau}\right\Vert \leq1$ since $\overrightarrow{\tau}$
is the hidden state of an LSTM. Hence, we have $\forall i,j$: $0<\frac{1}{2+\epsilon}\leq\left\langle \right\rangle _{ij}\leq\frac{1}{\epsilon}$.
Hence, let $B_{ij}\left(n\right)$ a random variable denoting $\frac{\left\langle \right\rangle _{ij}\left(n\right)}{\sum_{b\in\mathcal{N}_{j}^{K_{w}}}\left\langle \right\rangle _{bj}\left(n\right)}$--the
neighbor weight at step $n$, $\forall i,j:$

\[
\frac{\epsilon}{K_{w}\epsilon+2K_{w}-2}\leq B_{ij}\left(n\right)\leq\frac{2+\epsilon}{K_{w}\epsilon+2}
\]
That yields $\sum_{n=1}^{\infty}\lambda\left(n\right)\geq\frac{\epsilon}{K_{w}\epsilon+2K-2}\sum_{n=1}^{\infty}\alpha_{w}\left(n\right)$
and $\sum_{n=1}^{\infty}\lambda^{2}\left(n\right)\leq\left(\frac{2+\epsilon}{K_{w}\epsilon+2}\right)^{2}\sum_{n=1}^{\infty}\alpha_{w}^{2}\left(n\right)$.
Hence the writing updates converge when $\sum_{n=1}^{\infty}\alpha_{w}\left(n\right)=\infty$
and $\sum_{n=1}^{\infty}\alpha_{w}^{2}\left(n\right)<\infty$. We
can always choose such $\alpha_{w}$ (e.g., $\alpha_{w}\left(n\right)=\frac{1}{n+1}$). 

With a constant writing rate $\alpha$, we rewrite Eq. \ref{eq:write2}
as

\begin{align*}
\mathcal{M}_{i}^{v}\left(n+1\right) & =\mathcal{M}_{i}^{v}\left(n\right)+\alpha B_{ij}\left(n\right)\left(R_{j}\left(n\right)-\mathcal{M}_{i}^{v}\left(n\right)\right)\\
 & =\alpha B_{ij}\left(n\right)R_{j}\left(n\right)+\text{\ensuremath{\mathcal{M}_{i}^{v}\left(n\right)\left(1-\alpha B_{ij}\left(n\right)\right)}}\\
 & =\sum_{t=1}^{n}\alpha B_{ij}\left(t\right)\prod_{l=t+1}^{n}\left(1-\alpha B_{ij}\left(l\right)\right)R_{j}\left(t\right)\\
 & +\prod_{t=1}^{n}\left(1-\alpha B_{ij}\left(t\right)\right)\mathcal{M}_{i}^{v}\left(1\right)
\end{align*}
where the second term $\prod_{t=1}^{n}\left(1-\alpha B_{ij}\left(t\right)\right)\mathcal{M}_{i}^{v}\left(1\right)\rightarrow0$
as $n\rightarrow\infty$ since $B_{ij}\left(t\right)$ and $\alpha$
are bounded between 0 and 1. The first term can be decomposed into
three terms

\begin{align*}
\sum_{t=1}^{n}\alpha B_{ij}\left(t\right)\prod_{l=t+1}^{n}\left(1-\alpha B_{ij}\left(l\right)\right)R_{j}\left(t\right) & =T_{1}+T_{2}+T_{3}
\end{align*}

where 
\begin{align*}
T_{1} & =\sum_{t=1}^{n}\alpha B_{ij}\left(t\right)\prod_{l=t+1}^{n}\left(1-\alpha B_{ij}\left(l\right)\right)V_{i}\\
T_{2} & =\sum_{t=1}^{n}\alpha B_{ij}\left(t\right)\prod_{l=t+1}^{n}\left(1-\alpha B_{ij}\left(l\right)\right)\Delta V_{ij}\left(t\right)\\
T_{3} & =\sum_{t=1}^{n}\alpha B_{ij}\left(t\right)\prod_{l=t+1}^{n}\left(1-\alpha B_{ij}\left(l\right)\right)\tilde{R}_{j}\left(t\right)
\end{align*}
Here,$V_{i}$ is the true value of the trajectory stored in slot $i$,
$\Delta V_{ij}\left(t\right)=V_{j}\left(t\right)-V_{i}$ and $\tilde{R}_{j}\left(t\right)=R_{j}\left(t\right)-V_{j}\left(t\right)$
the noise term between the return and the true value. Assume that
the value is associated with zero mean noise and the value noise is
independent with the neighbor weights, then $\mathbb{E}\left(T_{3}\right)=0$\footnote{This assumption is true for the Perturbed Cart Pole Gaussian reward
noise. }. 

Further, we make other two assumptions: (1) the neighbor weights are
independent across update steps; (2) the probability $p_{j}$ of visiting
a neighbor $j$ follows the same distribution across update steps
and thus, $\mathbb{E}\left(B_{ij}\left(t\right)\right)=\mathbb{E}\left(B_{ij}\left(l\right)\right)=\mathbb{E}\left(B_{ij}\right)$.
We now can compute 

\begin{align*}
\mathbb{E}\left(T_{1}\right) & =\mathbb{E}\left(\sum_{t=1}^{n}\alpha B_{ij}\left(t\right)\prod_{l=t+1}^{n}\left(1-\alpha B_{ij}\left(l\right)\right)V_{i}\right)\\
 & =V_{i}\sum_{t=1}^{n}\alpha\mathbb{E}\left(B_{ij}\right)\prod_{l=t+1}^{n}\left(1-\alpha\mathbb{E}\left(B_{ij}\right)\right)\\
 & =V_{i}\alpha\mathbb{E}\left(B_{ij}\right)\sum_{t=1}^{n}\left(1-\alpha\mathbb{E}\left(B_{ij}\right)\right)^{n-t}\\
 & =V_{i}\alpha\mathbb{E}\left(B_{ij}\right)\frac{1-\left(1-\alpha\mathbb{E}\left(B_{ij}\right)\right)^{n}}{1-\left(1-\alpha\mathbb{E}\left(B_{ij}\right)\right)}\\
 & =V_{i}\left(1-\left(1-\alpha\mathbb{E}\left(B_{ij}\right)\right)^{n}\right)
\end{align*}
As $n\rightarrow\infty$, $\mathbb{E}\left(T_{1}\right)\rightarrow V_{i}$
since since $B_{ij}\left(t\right)$ and $\alpha$ are bounded between
0 and 1. 

Similarly, $\mathbb{E}\left(T_{3}\right)=\mathbb{E}\left(V_{j}\left(t\right)-V_{i}\right)=\mathbb{E}\left(V_{j}\left(t\right)\right)-V_{i}=\sum_{j\in\mathcal{N}_{i}^{K_{w}}}p_{j}V_{j}-V_{i}$,
which is the approximation error of the KNN algorithm. Hence, with
constant learning rate, on average, the $\memwrite$ operator leads
to the true value plus the approximation error of KNN. The quality
of KNN approximation determines the mean convergence of $\memwrite$
operator. Since the bias-variance trade-off of KNN is specified by
the number of neighbors $K$, choosing the right $K>1$ (not too big,
not too small) is important to achieve good writing to ensure fast
convergence. That explains why our writing to multiple slots ($K>1$)
is generally better than the traditional writing to single slot ($K=1$).

\subsubsection{Convergence Analysis of $\mathrm{refine}$ Operator\label{subsec:Convergence-Analysis-of-1}}

In this section, we study the convergence of the memory-based value
estimation by applying $\mathrm{refine}$ operator to the memory.
As such, we treat the $\memread\left(\overrightarrow{\tau}|\mathcal{M}\right)$
operator as a value function over trajectory space $\mathcal{T}$
and simplify the notation as $\memread\left(x\right)$ where $x$
represents the trajectory. We make the assumption that the $\memread$
operator simply uses averaging rule and the set of neighbors stored
in the memory is fixed (i.e. no new element is added to the memory)
, then 

\begin{eqnarray*}
\mathrm{\memread}_{t}\left(x\right) & = & \sum_{i\in\mathcal{N}^{K}(x)}B_{ix}\mathcal{M}_{i,t}^{v}
\end{eqnarray*}
where $B_{ix}=\frac{\left\langle \mathcal{M}_{i}^{k},x\right\rangle }{\sum_{j\in\mathcal{N}^{K}\left(x\right)}\left\langle \mathcal{M}_{j}^{k},x\right\rangle }$
is the neighbor weight and $t$ is the step of updating.

We rewrite the $\mathrm{refine}$ operator as

\begin{align*}
\mathcal{M} & \leftarrow\memwrite\left(x,\underset{a}{\max}\,r_{\varphi}\left(x,a\right)+\gamma\memread_{t}\left(y\right)|\mathcal{M}\right)\\
\Leftrightarrow\forall i\in\mathcal{N}^{K}(x) & :\\
\mathcal{M}_{i,t+1}^{v} & =\mathcal{M}_{i,t}^{v}\\
 & +\alpha_{w,t}B_{ix}\left(\underset{a}{\max}\,r_{\varphi}\left(x,a\right)+\gamma\memread_{t}\left(y\right)-\mathcal{M}_{i,t}^{v}\right)
\end{align*}
where $y$ is the trajectory after taking action $a$ from the trajectory
$x$. Then, after the $\mathrm{refine}$, 

\begin{eqnarray*}
\mathrm{\memread}_{t+1}\left(x\right) & = & \sum_{i\in\mathcal{N}^{K}(x)}B_{ix}\mathcal{M}_{i,t+1}^{v}\\
 & = & \sum_{i\in\mathcal{N}^{K}(x)}\mathcal{M}_{i,t}^{v}B_{ix}\left(1-\alpha_{w,t}\right)\\
 &  & +\alpha_{w,t}\sum_{i\in\mathcal{N}^{K}(x)}\left(\underset{a}{\max}\,r_{\varphi}\left(x,a\right)+\gamma\memread_{t}\left(y\right)\right)B_{ix}^{2}\\
 &  & +\alpha_{w,t}\sum_{i\in\mathcal{N}^{K}(x)}\mathcal{M}_{i,t}^{v}B_{ix}\left(1-B_{ix}\right)\\
 & = & \mathrm{\memread}_{t}\left(x\right)\left(1-\alpha_{w,t}\right)+\alpha_{w,t}G_{t}\left(x\right)
\end{eqnarray*}
where $G_{t}\left(x\right)=\underset{a}{\max}\,U_{t}\left(x,a\right)\sum_{i\in\mathcal{N}^{K}(x)}B_{ix}^{2}+\sum_{i\in\mathcal{N}^{K}(x)}\mathcal{M}_{i,t}^{v}B_{ix}\left(1-B_{ix}\right)$,
$U\left(x,a\right)=r_{\varphi}\left(x,a\right)+\gamma\memread\left(y\right)$.
To simplify the analysis, we assume the stored neighbors of $x$ are
apart from $x$ by the same distance, i.e., $\forall i\in\mathcal{N}^{K}(x):B_{ix}=\frac{1}{K}$.
That is,

\[
G\left(x\right)=\left(\underset{a}{\max}\,r_{\varphi}\left(x,a\right)+\gamma\memread\left(y\right)\right)\frac{1}{K}+\memread\left(x\right)\frac{K-1}{K}
\]

Let $\mathrm{H}$ an operator defined for the function $\memread:\mathcal{T\rightarrow}\mathbb{R}$
as

\begin{align*}
\mathrm{H}\memread\left(x\right) & =\sum_{\hat{y}\in\mathcal{T}}P_{a^{*}}\left(x,y\right)G_{t}\left(x|a^{*}\right)
\end{align*}
where $a^{*}=\underset{a}{\argmax}\sum_{\hat{y}\in\mathcal{T}}P_{a}\left(x,y\right)U\left(x,a\right)$.
We will prove $\mathrm{H}$ is a contraction in the sup-norm. 

Let us denote $\Delta\memread\left(x\right)=\memread_{1}\left(x\right)-\memread_{2}\left(x\right)$,
$\Delta PU\left(x,a_{1}^{*},a_{2}^{*}\right)=\sum_{y\in\mathcal{T}}P_{a_{1}^{*}}\left(x,y\right)U_{1}\left(x,a_{1}^{*}\right)$
$-\sum_{y\in\mathcal{T}}P_{a_{2}^{*}}\left(x,y\right)U_{2}\left(x,a_{2}^{*}\right)$
and $\hat{a}=\underset{a_{1}^{*},a_{2}^{*}}{\argmax}\sum_{y\in\mathcal{T}}P_{a_{1}^{*}}\left(x,y\right)U_{1}\left(x,a_{1}^{*}\right)$,
$\sum_{y\in\mathcal{T}}P_{a_{2}^{*}}\left(x,y\right)U_{2}\left(x,a_{2}^{*}\right)$.
Then,

\begin{align*}
\left\Vert \mathrm{H}\memread_{1}-\mathrm{H}\memread_{2}\right\Vert _{\infty} & =\left\Vert \sum_{y\in\mathcal{T}}\left(\Delta PU\left(x,a_{1}^{*},a_{2}^{*}\right)\frac{1}{K}\right.\right.\\
 & +\left.\left.\Delta\memread\left(x\right)\frac{K-1}{K}\right)\right\Vert _{\infty}\\
 & \leq\left\Vert \sum_{y\in\mathcal{T}}\Delta PU\left(x,a_{1}^{*},a_{2}^{*}\right)\frac{1}{K}\right\Vert _{\infty}\\
 & +\left\Vert \frac{K-1}{K}\Delta\memread\left(x\right)\right\Vert _{\infty}\\
 & \leq\left\Vert \sum_{y\in\mathcal{T}}\frac{P_{\hat{a}}\left(x,y\right)}{K}\gamma\Delta\memread\left(y\right)\right\Vert _{\infty}\\
 & +\left\Vert \frac{K-1}{K}\Delta\memread\left(x\right)\right\Vert _{\infty}\\
 & \leq\sum_{y\in\mathcal{T}}\frac{P_{\hat{a}}\left(x,y\right)}{K}\gamma\left\Vert \Delta\memread\left(y\right)\right\Vert _{\infty}\\
 & +\left\Vert \frac{K-1}{K}\Delta\memread\right\Vert _{\infty}\\
 & \leq\frac{\gamma+K-1}{K}\left\Vert \memread_{1}-\memread_{2}\right\Vert _{\infty}
\end{align*}
Since $\gamma<1$, $0<\gamma_{K}=\frac{\gamma+K-1}{K}<1\:\forall K\geq1$.
Thus, $\mathrm{H}$ is a contraction in the sup-norm and there exists
a fix-point $\memread^{*}$ such that $\mathrm{H}\memread^{*}=\memread^{*}$. 

We define $\Delta_{t}=\memread_{t}-\memread^{*}$, then

\[
\Delta_{t+1}=\Delta_{t}\left(x\right)\left(1-\alpha_{w,n}\right)+\alpha_{w,t}F_{t}\left(x\right)
\]
where $F_{t}\left(x\right)=G_{t}\left(x\right)-\memread^{*}\left(x\right)$.
We have

\begin{align*}
\mathbb{E}\left(F_{t}\left(x\right)\mid F_{t}\right) & =\sum_{\hat{y}\in\mathcal{T}}P_{a^{*}}\left(x,y\right)G_{t}\left(x|a^{*}\right)-\memread^{*}\left(x\right)\\
 & =\mathrm{H}\memread_{t}\left(x\right)-\memread^{*}\left(x\right)
\end{align*}
Following the proof in \cite{melo2001convergence}, $\mathbb{E}\left(F_{t}\left(x\right)\mid F_{t}\right)\leq\gamma_{K}\left\Vert \Delta_{t}\left(x\right)\right\Vert _{\infty}$
and $\mathrm{var}\left(F_{t}\left(x\right)\mid F_{t}\right)$<$C\left(1+\left\Vert \Delta_{t}\left(x\right)\right\Vert \right)^{2}$for
$C>0$. Assume that $\sum_{t=1}^{\infty}\alpha_{w,t}=\infty$ and
$\sum_{t=1}^{\infty}\alpha_{w,t}^{2}<\infty$, according to \cite{jaakkola1994convergence},
$\Delta_{t}$ converges to 0 with probability 1 or $\memread$ converges
to $\memread^{*}$.

\subsection{Experimental Details \label{subsec:Experimental-Details}}

\begin{table}
\begin{centering}
\begin{tabular}{cccc}
\hline 
Task & MBEC & MBEC++ & DQN\tabularnewline
\hline 
2D Maze & 2K & N/A & 43K\tabularnewline
Classical control & N/A & 39K & 43K\tabularnewline
Atari games & N/A & 13M & 13M\tabularnewline
3D Navigation & N/A & 13M & 13M\tabularnewline
\hline 
\end{tabular}
\par\end{centering}
\caption{The number of trainable parameters of MBEC(++) and its main competitor
DQN.\label{tab:Number-of-trainable}}

\end{table}
\begin{figure}
\begin{centering}
\includegraphics[width=1\columnwidth]{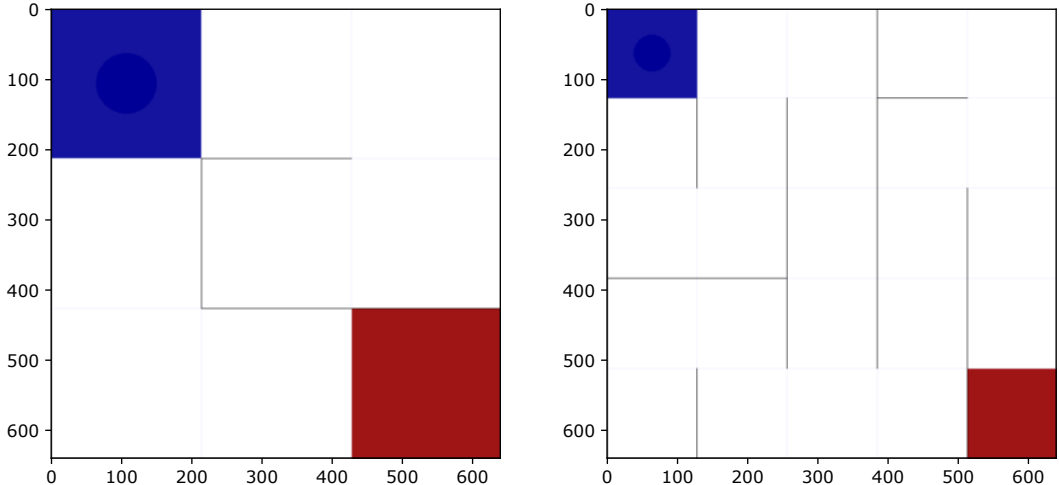}
\par\end{centering}
\caption{Maze map $3\times3$ (left) and $5\times5$ (right). The agent starts
from the left top corner (blue square) and finds the goal at the right
bottom corner (red square).\label{fig:mazemap}}
\end{figure}

\subsubsection{Implemented baseline description}

In this section, we describe baselines that are implemented and used
in our experiments. We trained all the models using a single GPU Tesla
V100-SXM2. 

\paragraph{Model-based Episodic Control  (MBEC, ours)}

The main hyper-parameter of MBEC is the number of neighbors ($K$),
chunk length ($L$) and memory slots ($N$). Hyper-parameter tuning
is adjusted according to specific tasks. For example, for small and
simple problems, $K$ and $N$ tend to be smaller and $L$ is often
about $20-30\%$ of the average length per episode. Across experiments,
we follow prior works using $\gamma=0.99$. We also fix $\alpha_{w}=0.5$
to reduce hyperparameter tuning. To implement $\memread$ and $\memwrite$,
we set $K=K_{r}=K_{w}$ and use the kernel $\left\langle x,y\right\rangle =\frac{1}{\left\Vert x-y\right\Vert +\epsilon}$
with $\epsilon=10^{-3}$ following \cite{pritzel2017neural}. 

Unless stated otherwise, the hidden size of the trajectory model is
fixed to $H=16$ for all tasks. The reward model is implemented as
a 2-layer ReLU feed-forward neural network and trained with batch
size 32 for all tasks. To compute TR Loss, we sample 4 past transitions
and add Gaussian noise (mean 0, std. $0.1\times\left\Vert q\right\Vert _{2}$)
to the query vector $q$. Notably, in MBEC++, when training with TD
loss, we do not back-propagate the gradient to the trajectory model
to ensure that the trajectory representations are only shaped by the
TR Loss. 

In practice, to reduce computational complexity, we do not perform
$\mathrm{refine}$ operator every timestep. Rather, at each step,
we randomly $\mathrm{refine}$ with a probability $p_{u}=0.1$. Similarly,
we occasionally update the parameters of the trajectory model. Every
$L$ step, we randomly update $\phi$ and $\omega$ using back-propagation
via $\mathcal{L}_{tr}$ with probability $p_{rec}=0.5$. For Atari
games, we stop training the trajectory model after 5 million steps.
On our machine for Atari games, these tricks generally make MBEC++
run at speed 100 steps/s while DQN 150 steps/s.

\paragraph{Deep Q-Network (DQN)}

Except for Sec. \ref{subsec:Atari-2600-Benchmark} and \ref{subsec:POMDP:-3D-Navigation},
we implement DQN\footnote{\url{https://github.com/higgsfield/RL-Adventure}}
with the following hyper-parameters: 3-layer ReLU feed-forward Q-network
(target network) with hidden size 144, target update every 100 steps,
TD update every 1 step, replay buffer size $10^{6}$ and Adam optimizer
with batch size 32. The exploration rate decreases from $1$ to $0.01$.
We tune the learning rate for each task in range $\left[10^{-3},10^{-5}\right]$.
For tasks with image input, the Q-network (target network) is augmented
with CNN to process the image depending on tasks. MBEC++ adopts the
same DQN with a smaller hidden size of 128. Table \ref{tab:Number-of-trainable}
compares model size between DQN and MBEC(++). Regarding memory usage,
for Atari games, DQN consumes 1,441 MB and MBEC++ 1,620 MB. 

\paragraph{Model-Free Episodic Control (MFEC)}

This episodic memory maintains a value table using $K$-nearest neighbor
to read the value for a query state and $\max$ operator to write
a new value. We set the key dimension and memory size to 64 and $10^{6}$,
respectively. We tune $K\in\left\{ 3,5,11,25\right\} $ for each task.
Unless stated otherwise, we use random projection for MFEC. For VAE-CNN
version used in dynamic maze mode, we use 5-convolutional-layer encoder
and decoder (16-128 kernels with 4\texttimes 4 kernel size and a stride
of 2). Other details follow the original paper \cite{blundell2016model}.

\paragraph{Neural Episodic Control (NEC)}

This model extends MFEC using the state-key mapping as a CNN embedding
network trained to minimize the TD error of memory-based value estimation.
Also, multi-step Q-learning update is employed for memory writing.
We adopt the publicly available source code \footnote{\url{https://github.com/hiwonjoon/NEC}}
which follows the same hyper-parameters used in the original paper
\cite{pritzel2017neural} and apply it to stochastic control problem
by implementing the embedding network as a 2-layer feed-forward neural
network. We tune $K\in\left\{ 3,5,11,25\right\} $ and the hidden
size of the embedding network $\in\left\{ 32,64,128,256\right\} $
for each task. 

\paragraph{Proximal Policy Optimization (PPO)}

PPO \cite{schulman2017proximal} is a policy gradient method that
simplifies Trust Region update with gradient descent and soft constraint
(maintaining low KL divergence between new and old policy via objective
clipping). We test PPO for the 3D Navigation task using the original
source code of the environment Gym Mini World. 

\paragraph{Deep Recurrent Q-Network (DRQN)}

DRQN \cite{hausknecht2015deep} is similar to DQN except that it uses
LSTM as the Q-Network. As the hidden state of LSTM represents the
environment state for the Q-Network, it captures past information
that may be necessary for the agent in POMDP. We extend DQN to DRQN
by storing transitions with the hidden states in the replay buffer
and replacing the feed-forward Q-Network with an LSTM Q-Network. We
tune the hidden size of the LSTM $\in\left\{ 128,256,512\right\} $
for 3D navigation task.

\subsubsection{Maze task \label{subsec:Maze-task}}

\paragraph{Task overview}

In the maze task, if the agent hits the wall of the maze, it gets
$-1$ reward. If it reaches the goal, it gets $1$ reward. For each
step in the maze, the agent get $-0.1/n_{e}^{2}$ reward. An episode
ends either when the agent reaches the goal or the number of steps
exceeds 1000. 

To build different modes of the task, we modify the original gym-maze
environment\footnote{\url{https://github.com/MattChanTK/gym-maze}}.
Fig. \ref{fig:mazemap} illustrates the original $3\times3$ and $5\times5$
maze structure. We train and tune MBEC and other baselines for $3\times3$
maze task and use the found hyper-parameters for other task modes.
For MBEC, the best hyper-parameters are $K=5$, $L=5$, $N=1000$.

\paragraph{Transition Prediction (TP) loss}

For dynamic mode and ablation study for stochastic control tasks,
we adopt a common loss function to train the traditional model in
model-based RL: the transition prediction (TP) loss. Trained with
the TP loss, the model tries to predict the next observations given
current trajectory and observations. The TP loss is concretely defined
as follows,

\begin{align}
\mathcal{L}_{tp} & =\mathbb{E}\left(\left\Vert y^{*}\left(t\right)-\left[s_{t+1},a_{t+1}\right]\right\Vert _{2}^{2}\right)\label{eq:trj_l-1}\\
y^{*}\left(t\right) & =\mathcal{G}_{\omega}\left(\mathcal{T_{\phi}}\left(\left[\tilde{s}_{t},\tilde{a}_{t}\right],\overrightarrow{\tau}_{t-1}\right)\right)\label{eq:rec-1}
\end{align}
The key difference between TP loss and TR loss is the timestep index.
TP loss takes observations at current timestep to predict the one
at the next timestep. On the other hand, TR loss takes observations
at past timestep and uses the current working memory (hidden state
of the LSTM) to reconstruct the observations at the timestep after
the past timestep. Our experiments consistently show that TP loss
is inferior to our proposed TR loss (see Sec. \ref{subsec:app_scc}). 
\begin{center}
\par\end{center}

\begin{figure*}
\begin{centering}
\includegraphics[width=0.9\textwidth]{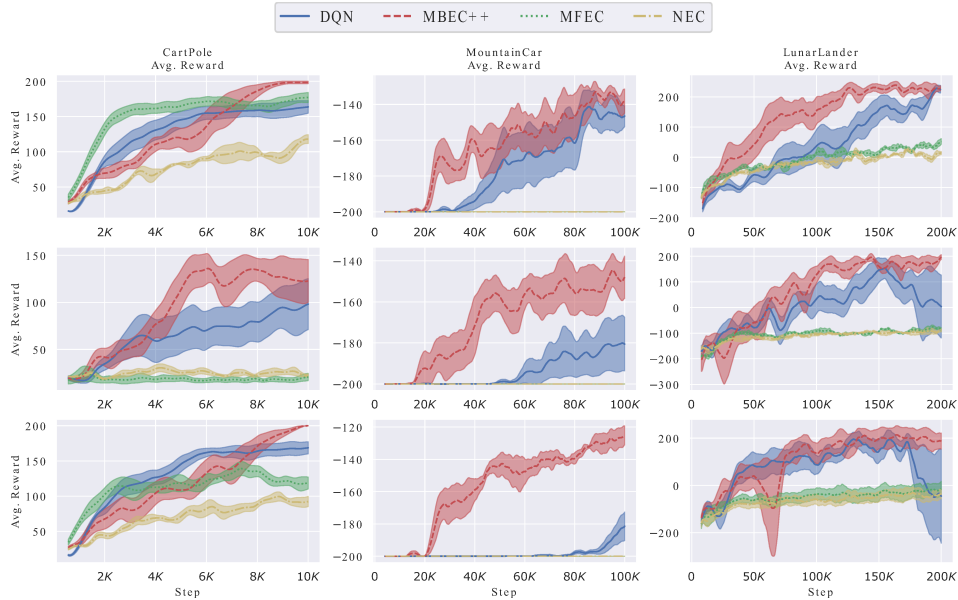}
\par\end{centering}
\caption{All learning curves for stochastic classical control task. First row:
Gaussian noisy reward. Second row: Bernoulli noisy reward. Third row:
Noisy transition \label{fig:all9classic}}
\end{figure*}

\begin{center}
\begin{figure*}
\begin{centering}
\includegraphics[width=1\textwidth]{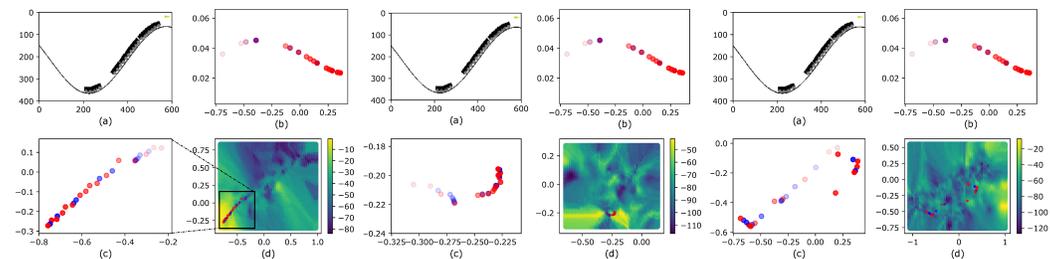}
\par\end{centering}
\caption{Stochastic Mountain Car. Trajectory space in noisy transition using
models trained with TR loss (left), no training (middle) and TP loss
(right). \label{fig:mc_vc_noise}}
\end{figure*}
\par\end{center}

\subsubsection{Stochastic classical control task \label{subsec:app_scc}}

\paragraph{Task description}

We introduce three ways to make a classical control problem stochastic.
First, we add Gaussian noise (mean 0, $\sigma_{re}=0.2$) to the reward
that the agent observes. Second, we add Bernoulli-like noise (with
a probability $p_{re}=0.2$, the agent receives a reward $-r$ where
$r$ is the true reward). Finally, we make the observed transition
noisy by letting the agent observe the same state despite taking any
action with a probability $p_{tr}=0.5$. The randomness only affects
what the agent sees while the environment dynamic is not affected.
Three classical control problems are chosen from Open AI's gym: CartPole-v0,
MountainCar-v0 and LunarLander-v2. For each problem, we apply the
three stochastic configurations, yielding 9 tasks in total. 

Fig. \ref{fig:all9classic} showcases the learning curves of DQN,
MBEC++, MFEC and NEC for all 9 tasks. MBEC++ is consistently the leading
performer. DQN is often the runner-up, yet usually underperforms our
method by a significant margin. Overall, other memory-based methods
such as MFEC and NEC perform poorly for these tasks since they are
not designed for stochastic environments.

\paragraph{Memory contribution }

We determine the episodic and semantic contribution to the final value
estimation by counting the number of times their greedy actions equal
the final greedy action, dividing by the number of timesteps. Concretely,
the episodic and semantic contribution is computed respectively as

\[
\frac{\sum_{t=1}^{T}\argmax_{a_{t}}Q_{eps}\left(s_{t},a_{t}\right)==\argmax_{a_{t}}Q\left(s_{t},a_{t}\right)}{T}
\]

\[
\frac{\sum_{t=1}^{T}\argmax_{a_{t}}Q_{\theta}\left(s_{t},a_{t}\right)==\argmax_{a_{t}}Q\left(s_{t},a_{t}\right)}{T}
\]
where $Q_{eps}\left(s_{t},a_{t}\right)=Q_{MBEC}\left(s_{t},a_{t}\right)f_{\beta}\left(s_{t},\overrightarrow{\tau}_{t-1}\right)$,
$Q_{\theta}\left(s_{t},a_{t}\right)$ and $Q\left(s_{t},a_{t}\right)$
represent the episodic, semantic and final value estimation, respectively.

Fig. \ref{fig:mcontrib} illustrates the running average contribution
using a window of 100 timesteps. We note that the contribution of
the two does not need to sum up to 1 as both can agree with the same
greedy action. 
\begin{center}
\begin{figure*}
\begin{centering}
\includegraphics[width=0.9\textwidth]{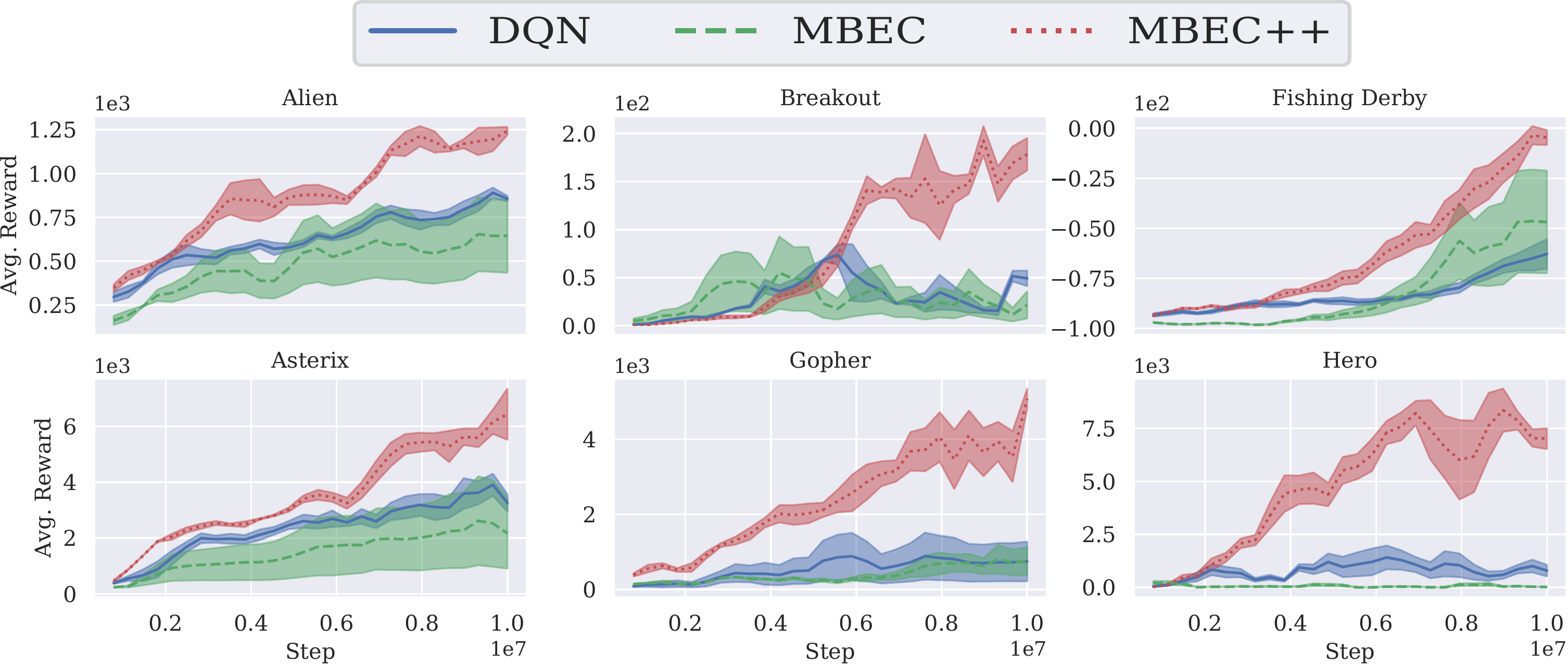}
\par\end{centering}
\caption{Learning curves of several Atari games.\label{fig:rl6}}
\end{figure*}
\par\end{center}

\begin{center}
\begin{figure}
\begin{centering}
\includegraphics[width=1\columnwidth]{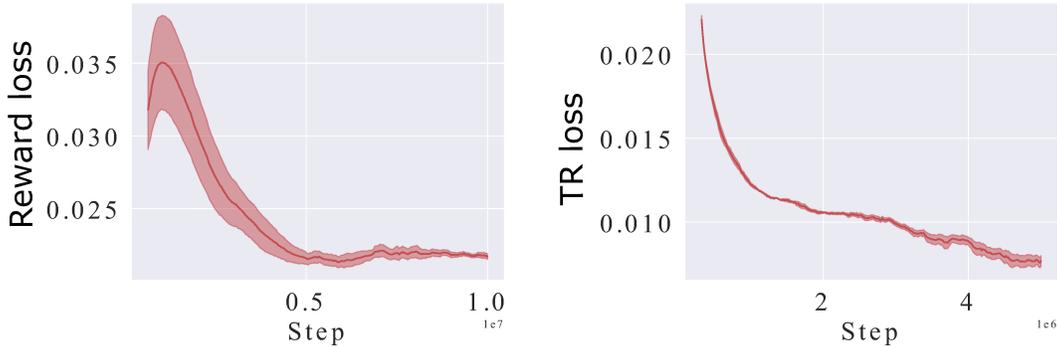}
\par\end{centering}
\caption{Loss of reward model and TR loss over training iterations for Atari's
Asterix and Gopher, respectively.\label{fig:atari_l2}}
\end{figure}
\begin{table*}
\begin{centering}
\begin{tabular}{ccccc}
\hline 
Game & Nature DQN & MFEC & NEC & MBEC++\tabularnewline
\hline 
Alien & 634.8  & 1717.7  & \textbf{3460.6} & 1991.2\tabularnewline
Amid & 126.8  & 370.9  & \textbf{811.3} & 369.0\tabularnewline
Assault & 1489.5  & 510.2  & 599.9 & \textbf{4981.3}\tabularnewline
Asterix & 2989.1  & 1776.6  & 2480.4 & \textbf{7724.0}\tabularnewline
Asteroids  & 395.3  & \textbf{4706.8 } & 2496.1 & 1456.2\tabularnewline
Atlantis  & 14210.5  & 95499.4  & 51208.0 & 99270.0\tabularnewline
Bank Heist  & 29.3 & 163.7 &  343.3 & 1126.4\tabularnewline
Battlezone  & 6961.0  & 19053.6  & 13345.5 & 30004.0\tabularnewline
Beamrider & 3741.7 & 858.8  & 749.6 & \textbf{5875.2}\tabularnewline
Berzerk  & 484.2 & \textbf{924.2} &  852.8 & 759.2\tabularnewline
Bowling  & 35.0 & 51.8 &  71.8 & \textbf{80.6}\tabularnewline
Boxing  & 31.3 & 10.7  & 72.8 & \textbf{95.8}\tabularnewline
Breakout  & 36.8  & 86.2  & 13.6 & \textbf{372.2}\tabularnewline
Centipede  & 4401.4  & \textbf{20608.8} &  12314.5 & 8693.8\tabularnewline
Chopper Command  & 827.2 & 3075.6 & \textbf{5070.3} & 1694.0\tabularnewline
Crazy Climber & 66061.6  & 9892.2  & 34344.0 & \textbf{107740.0}\tabularnewline
Defender & 2877.90 & 10052.80 & 6126.10 & \textbf{690956.0}\tabularnewline
Demon Attack  & 5541.9  & 1081.8  & 641.4 & 8066.4\tabularnewline
Double Dunk  & -19.0 & -13.2 & \textbf{ 1.8} & -1.8\tabularnewline
Enduro  & \textbf{364.9 } & 0.0  & 1.4 & 343.7\tabularnewline
Fishing Derby  & -81.6 & -90.3 &  -72.2 & \textbf{17.6}\tabularnewline
Freeway  & 21.5  & 0.6  & 13.5 & \textbf{33.1}\tabularnewline
Frostbite  & 339.1  & 925.1 & \textbf{ 2747.4 } & 1783.0\tabularnewline
Gopher  & 1111.2  & 4412.6  & 2432.3 & \textbf{11386.4}\tabularnewline
Gravitar  & 154.7  & 1011.3 & \textbf{ 1257.0} & 428.0\tabularnewline
H.E.R.O.  & 1050.7  & 14767.7  & \textbf{16265.3} & 12148.5\tabularnewline
Ice Hockey  & -4.5 & -6.5 &  -1.6 & \textbf{-1.5}\tabularnewline
James Bond  & 165.9  & 244.7 &  376.8 & \textbf{898.0}\tabularnewline
Kangaroo  & 519.6  & 2465.7  & 2489.1 & \textbf{16464.0}\tabularnewline
Krull  & 6015.1  & 4555.2  & 5179.2 & \textbf{9031.38}\tabularnewline
Kung Fu Master  & 17166.1  & 12906.5  & 30568.1 & \textbf{37100.0}\tabularnewline
Montezuma\textquoteright s Revenge  & 0.0  & \textbf{76.4 } & 42.1 & 0.0\tabularnewline
Ms. Pac-Man  & 1657.0  & 3802.7  & \textbf{4142.8} & 2687.2\tabularnewline
Name This Game  & 6380.2  & 4845.1 &  5532.0 & \textbf{7822.8}\tabularnewline
Phoenix  & 5357.0 & 5334.5  & 5756.5 & \textbf{15051.8}\tabularnewline
Pitfall!  & \textbf{0.0 } & -79.0  & \textbf{0.0} & \textbf{0.0}\tabularnewline
Pong  & -3.2 & -20.0 &  20.4 & \textbf{20.8}\tabularnewline
Private Eye  & 100.0  & \textbf{3963.8 } & 162.2 & 100.0\tabularnewline
Q{*}bert  & 2372.5  & \textbf{12500.4 } & 7419.2 & 8686.0\tabularnewline
River Raid  & 3144.9 & 4195.0 &  5498.1 & \textbf{10656.4}\tabularnewline
Road Runner  & 7285.4  & 5432.1  & 12661.4 & \textbf{55284.0}\tabularnewline
Robot Tank  & 14.6  & 7.3  & 11.1 & \textbf{23.9}\tabularnewline
Seaquest  & 618.7  & 711.6 & 1015.3 & \textbf{10460.2 }\tabularnewline
Skiing  & -19818.0  & -15278.9  & -26340.7 & -10016.0 \tabularnewline
Solaris  & 1343.0  & \textbf{8717.5 } & 7201.0 & 1692.0\tabularnewline
Space Invaders  & 642.2 & \textbf{2027.8 } & 1016.0 & 1425.6\tabularnewline
Stargunner  & 604.8  & 14843.9  & 1171.4 & 49640.0\tabularnewline
Tennis & 0.0  & -23.7  & -1.8 & \textbf{18.8}\tabularnewline
Time Pilot & 1952.0 & \textbf{10751.3 } & 10282.7 & 6752.0\tabularnewline
Tutankham  & 148.7  & 86.3  & 121.6 & 206.36\tabularnewline
Up\textquoteright n Down  & 18964.9  & 22320.8 & \textbf{ 39823.3} & 21743.2\tabularnewline
Venture  & 3.8  & 0.0 &  0.0 & \textbf{1092.4}\tabularnewline
Video Pinball  & 14316.0  & 90507.7 &  22842.6 & \textbf{182887.9}\tabularnewline
Wizard of Wor  & 401.4  & \textbf{12803.1 } & 8480.7 & 6252.0\tabularnewline
Yars\textquoteright{} Revenge  & 7614.1  & 5956.7  & 21490.5 & \textbf{21889.8}\tabularnewline
Zaxxon & 200.3  & 6288.1  & 10082.4 & \textbf{11180.0}\tabularnewline
\hline 
\end{tabular}
\par\end{centering}
\caption{Scores at 10 million frames.\label{tab:Scores-at-10}}
\end{table*}
\par\end{center}

\paragraph{Trajectory space visualization}

We visualize the memory-based value function w.r.t trajectory vectors
in Fig. \ref{fig:vis_mc_full} (d). As such, we set the trajectory
dimension to $H=2$ and estimate the value for each grid point (step
0.05) using $\memread$ operator as

\[
V\left(\overrightarrow{\tau}\right)\approx\sum_{i\in\mathcal{N}^{K_{r}}(\overrightarrow{\tau})}\frac{\left\langle \mathcal{M}_{i}^{k},\overrightarrow{\tau}\right\rangle \mathcal{M}_{i}^{v}}{\sum_{j\in\mathcal{N}^{K}(\overrightarrow{\tau})}\left\langle \mathcal{M}_{j}^{k},\overrightarrow{\tau}\right\rangle }
\]

To cope with noisy environment, MBEC relies on noise-tolerant trajectory
representation. As demonstrated in Fig. \ref{fig:mc_vc_noise}, even
when the state representations are disturbed by not changing to true
states, the trajectory representations shaped by the TR loss maintain
good approximation and interpolate well the latent location of disturbed
trajectories. In contrast, representations generated by random model
or model trained with TP loss fail to discover latent location of
disturbed trajectories, either collapsing (random model) or shattering
(TP loss model).

The failure of TP loss is understandable since it is very hard for
predicting the next transition when half of the ground truth is noisy
($p_{tr}=0.5$). On the other hand, TR loss facilitates easier learning
wherein the model only needs to reconstruct past observations which
are somehow already encoded in the current representation. 
\begin{center}
\begin{figure*}
\begin{centering}
\includegraphics[width=1\linewidth]{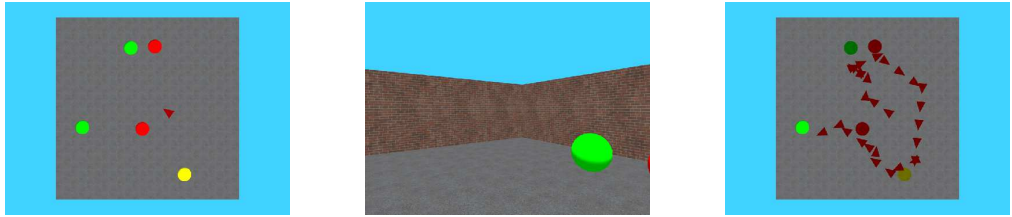}
\par\end{centering}
\caption{3D Navigation: picking 5 objects task. Top view map (left) and frontal
view-port or the observed state of the agent (middle) and a solution
found by MBEC++ (right).\label{fig:3dmap}}
\end{figure*}
\par\end{center}

\begin{figure}
\begin{centering}
\includegraphics[width=1\columnwidth]{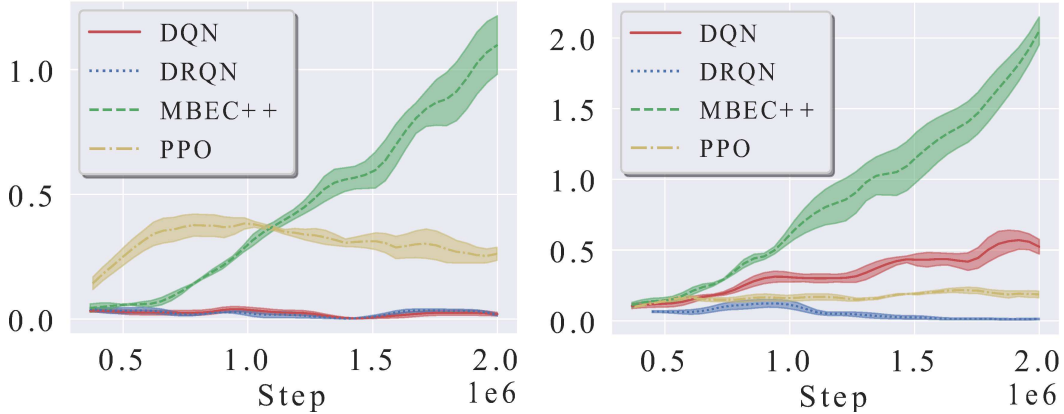}
\par\end{centering}
\caption{3D Navigation. Average reward for picking 3 (left) and 5 (right) objects
over environment steps (mean and std. over 5 runs). \label{fig:3d}}
\end{figure}

\subsubsection{Atari 2600 task \label{subsec:app_atari}}

We use Adam optimizer with a learning rate of $10^{-4}$, batch size
of 32 and only train the models within 10 million frames for sample
efficiency. Other implementations follows \cite{mnih2015human} (CNN
architecture, exploration rate, 4-frame stacking, reward clipping,
etc.). In our implementation, at timestep t, we use frames at t,t-1,t-2,t-3
and still count t as the current frame as well as the current timestep.
We tune the hyper-parameters for MBEC++ using the standard validation
procedure, resulting in $K=7$, $L=100$ and $N=50,000$. 

We also follow the training and validation procedure from \cite{mnih2015human}
using a public DQN implementation\footnote{\url{https://github.com/Kaixhin/Rainbow}}.
The CNN architecture is a stack of 4 convolutional layers with numbers
of filters, kernel sizes and strides of $\left[32,64,64,1024\right]$,$\left[8,4,3,3\right]$
and $\left[4,2,1,1\right]$, respectively. The Atari platform is Open
AI's Atari environments\footnote{\url{https://gym.openai.com/envs/atari}}.

In order to understand better the efficiency of MBEC++, we record
and compare the learning curves of MBEC++, MBEC and DQN (our implementation
using the same training procedure) in Fig. \ref{fig:rl6}. We run
the 3 models on 6 games (Alien, Breakout, Fishing Derby, Asterix,
Gopher and Hero), and plot the average performance over 5 random seeds.
We observe a common pattern for all learning curves, in which the
performance gap between MBEC++ and DQN becomes clearer around 5 million
steps and gets wider afterwards. We note that MBEC++ demonstrate fast
learning curves for some games (e.g., Breakout and Fishing Derby)
that other methods (DQN, MFEC or NEC) struggle to learn.

We realize that early stopping of training trajectory model or reward
model does not affect the performance much as the quality of trajectory
representations and reward prediction is acceptable at about 5 millions
steps (see Fig. \ref{fig:atari_l2}). Early stopping further accelerates
the running speed of MBEC++ and also helps stabilize the learning
of the Q-Networks. 

Table \ref{tab:Scores-at-10} reports the final testing score of MBEC++
and other baselines for all Atari games. We note that we only conducted
five runs for the Atari games mentioned in Fig. 9. For the remaining
games, our limited compute budget did not allow us to perform multiple
runs, and thus, we only ran once. We store the best MBEC++ models
based on validation score for each game during training and test them
for $100$ episodes. Other baselines' numbers are reported from \cite{pritzel2017neural}.
Compared to other baselines, MBEC++ is the winner on the leaderboard
for about half of the games. 

Notably, our episodic memory is much smaller than that of others.
For example, NEC and EMDQN maintain 5 millions slots per action (there
are total 18 actions). Meanwhile, our best number of memory slots
$N$ is only 50,000. 

\begin{table}
\begin{centering}
\begin{tabular}{ccccccc}
\hline 
Model & Alien & Asterix & Breakout & Fishing Derby & Gopher & Hero\tabularnewline
\hline 
Dreamer-v2 & \textbf{2950.1 } & 3100.8  & 57.0 & -13.6  & \textbf{16002.8} & \textbf{13552.9} \tabularnewline
Our MBEC++ & 1991.2  & \textbf{7724.0 } & \textbf{372.2 } & \textbf{17.6 } & 11386.4  & 12148.5\tabularnewline
\hline 
\end{tabular}
\par\end{centering}
\caption{Dreamver-v2 vs MBEC++ on 6 Atari games at 10M frames. We report the
best results of the models after three runs.\label{tab:Dreamver-v2-vs-MBEC++}}
\end{table}

\begin{table}
\begin{centering}
\begin{tabular}{ccccccc}
\hline 
{\scriptsize{}Model} & {\scriptsize{}Alien} & {\scriptsize{}Asterix} & {\scriptsize{}Breakout} & {\scriptsize{}Fishing Derby} & {\scriptsize{}Gopher} & {\scriptsize{}Hero}\tabularnewline
\hline 
{\scriptsize{}SIMPLE$^{\clubsuit}$} & \textbf{\scriptsize{}378.3 \textpm{} 85.5 } & {\scriptsize{}668.0 \textpm{} 294.1 } & {\scriptsize{}6.1 \textpm{} 2.8 } & {\scriptsize{}-94.5 \textpm{} 3.0 } & \textbf{\scriptsize{}510.2 \textpm{} 158.4}{\scriptsize{} } & {\scriptsize{}621.5 \textpm{} 1281.3}\tabularnewline
{\scriptsize{}Our MBEC++} & {\scriptsize{}340.5 \textpm{} 39.7 } & \textbf{\scriptsize{}810.1 \textpm{} 42.4 } & \textbf{\scriptsize{}11.9 \textpm{} 2.0}{\scriptsize{} } & \textbf{\scriptsize{}-81.5 \textpm{} 2.3}{\scriptsize{} } & {\scriptsize{}459.2 \textpm{} 60.4 } & \textbf{\scriptsize{}1992.8 \textpm{} 1171.9}\tabularnewline
\hline 
\end{tabular}
\par\end{centering}
\caption{SIMPLE vs MBEC++ on 6 Atari games at 400K frames. Mean and std over
5 runs. $\clubsuit$ is from \cite{kaiser2019model}. \label{tab:SIMPLE-vs-MBEC++-1}}
\end{table}

\subsubsection{3D navigation task \label{subsec:app_3dnav}}

In this task, the agent's goal is to pick objects randomly located
in a big room\footnote{\url{https://github.com/maximecb/gym-miniworld}}.
There are 5 possible actions (moving directions and object interaction)
and the number of objects is customizable. We train MBEC++, DQN and
DRQN using the same training procedure and CNN for state encoding
as in the Atari task. Except for DQN, other baselines (DRQN and PPO)
uses LSTM to encode the state of the environment. We also stack 4
consecutive frames to help the models (especially DQN) cope with the
environment's limited observation. For PPO, we tuned the clipping
threshold \{0.2, 0.5, 0.8\} and reported the best result (0.2). 

Fig. \ref{fig:3dmap} illustrates one sample of the environment map
and a solution found by MBEC++. The best hyper-parameters for MBEC++
are $K=15$, $L=20$ and $N=10000$.

\subsubsection{Ablation study\label{subsec:Ablation-study}}

\paragraph{Classical control}

We tune MBEC++ with Noisy Transition Mountain Car problem using range
$K\in\left\{ 1,5,15\right\} $, $L=\left\{ 1,10,50\right\} $, $N=\left\{ 500,3000,30000\right\} $.
We use the best found hyper-parameters ($K=15$, $L=10$, $N=3000$)
for all 9 problems. The learning curves of MBEC++ in Noisy Transition
Mountain Car are visualized in Fig. \ref{fig:mc_abl} (the first 3
plots). 

We also conduct an ablation study on MBEC++: (i) without TR loss,
(ii) with TP loss instead (iii), without multiple write ($K_{w}=1$)
and (iv) without memory refining. The result demonstrates that ablating
any components reduces the performance of MBEC++ significantly (see
Fig. \ref{fig:mc_abl} (the final plot)). 

\paragraph{Dynamic consolidation}

We compare our dynamic consolidation with traditional fixed combinations
in both simple and complex environments. Fixed combination baselines
use fixed $\beta$ in Eq. \ref{eq:qe-qs}, resulting in

\[
Q\left(s_{t},a_{t}\right)=Q_{MBEC}\left(s_{t},a_{t}\right)\beta+Q_{\theta}\left(s_{t},a_{t}\right)
\]

In CartPole (Gaussian noisy reward), all fixed combinations achieve
moderate results, yet fails to solve the task after 10,000 training
steps. Dynamic consolidation learning to generate dynamic $\beta$,
in contrast, completely solves the task (see Fig. \ref{fig:Ablation-studies-on}
(a)). 

In Atari game's Riverrraid--a more complex environment, the performance
gap between dynamic $\beta$ and fixed $\beta$ becomes clearer. Average
reward of dynamic consolidation reaches nearly 7,000 while the best
fixed combination's ($\beta=0.1$) is less than 3,000 (see Fig. \ref{fig:Ablation-studies-on}
(b)).

\begin{figure}
\begin{centering}
\includegraphics[width=1\columnwidth]{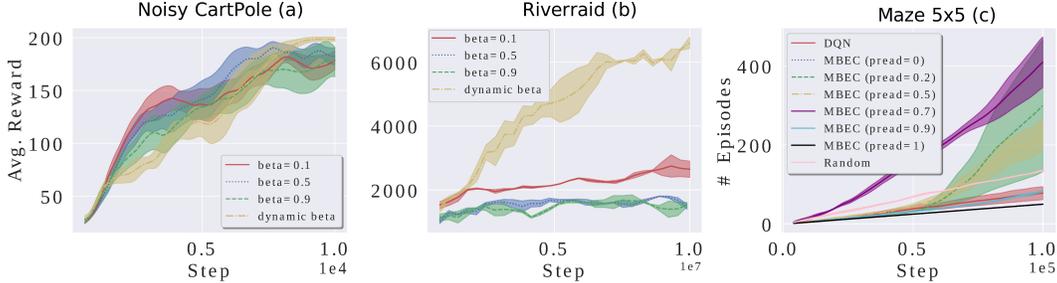}
\par\end{centering}
\caption{Ablation studies on dynamic consolidation (a,b) and $p_{read}$ (c).
All curves are reported with mean and std. over 5 runs. \label{fig:Ablation-studies-on}}
\end{figure}

\paragraph{Tuning $p_{read}$}

Besides modified modes introduced in the main manuscript, we investigate
MBEC with different $p_{read}$ and DQN in a bigger maze ($5\times5$)
for the original setting. As shown in Fig. \ref{fig:Ablation-studies-on}
(c), many MBEC variants successfully learn the task, significantly
better in the number of completed episodes compared to random and
DQN agents. We find that when $p_{read}=1$ (only average reading),
the performance of MBEC is negatively affected, which indicates that
the role of max reading is important. Similar situation is found for
for $p_{read}=0$ (only max reading). Among all variants, $p_{read}=0.7$
shows stable and fastest learning. Hence, we set $p_{read}=0.7$ for
all other experiments in this paper.

\end{document}